\documentclass[letterpaper]{article}
\usepackage{uai20}
\usepackage[margin=1in]{geometry}

\usepackage{times}
\usepackage{microtype}
\usepackage{graphicx}
\usepackage{subfigure}
\usepackage{booktabs} 
\usepackage{times}
\usepackage{subfigure} 
\usepackage{amsthm}
\usepackage{bm}
\usepackage{xfrac}
\usepackage{xcolor}
\usepackage{color}
\usepackage{cancel}
\usepackage{amssymb}
\usepackage{mathtools}
\usepackage{amsmath}
\usepackage[group-separator={,}]{siunitx}
\usepackage{natbib}
\usepackage{enumitem}

\newcommand{\Vis}{\mathcal{V}}
\newcommand{\UP}{\text{U}}

\newcommand{\vis}{v}
\newcommand{\hid}{h}
\newcommand{\Thett}{\widetilde{\Theta}}

\newcommand{\thetb}{\theta}
\newcommand{\thetbt}{\widetilde{\theta}}

\renewcommand{\P}{p}
\newcommand{\Pdata}{\P_{\text{\tiny data}}}

\DeclareMathOperator{\argmin}{\textnormal{argmin}}
\newcommand{\tr}{\text{tr}}
\newcommand{\defeq}{\vcentcolon=}

\newcommand{\Ts}{\widetilde{\Theta}}

\newcommand{\ts}{\tilde{\theta}}

\newcommand{\RR}{\mathbb{R}}
\newcommand{\EE}{\mathbb{E}}
\newcommand{\nll}[2]{\mathcal{L}(#1|\,#2)}

\newcommand{\x}{x}
\newcommand{\xh}{\hat{h}}
\newcommand{\z}{\bm{\tau}}
\newcommand{\zt}{\widetilde{\z}}
\newcommand{\zd}{\z^*}
\newcommand{\zp}{\z^\prime}
\newcommand{\w}{\omega}
\newcommand{\wt}{\widetilde{\omega}}
\newcommand{\at}{\widetilde{a}}
\newcommand{\pit}{\widetilde{\pi}}

\newcommand{\mut}{\widetilde{\mu}}
\newcommand{\mub}{\mu}
\definecolor{darkgreen}{rgb}{0.0, 0.5, 0.0}

\newcommand{\veps}{\varepsilon}
\newcommand{\normal}[2]{\mathcal{N}\big(#1,#2\big)}

\newcommand{\Vt}{\widetilde{V}}
\newcommand{\Qt}{\widetilde{Q}}
\newcommand{\Rt}{\widetilde{R}}
\newcommand{\At}{\widetilde{A}}
\newcommand{\Ct}{\widetilde{C}}

\newcommand{\Vb}{\hat{V}}
\newcommand{\Qb}{\hat{Q}}
\newcommand{\Rb}{\hat{R}}

\newtheorem{theorem}{Theorem}
\newtheorem{corollary}[theorem]{Corollary}
\newtheorem{lemma}[theorem]{Lemma}

\newtheorem{remark}[theorem]{Remark}

\usepackage{caption}
\usepackage{hyperref}

\title{Divergence-Based Motivation for Online EM\\
and Combining Hidden Variable Models}

\author{ {\bf Ehsan Amid and Manfred K. Warmuth} \\
UC Santa Cruz and Google Brain, Mountain View\\
\texttt{\{eamid, manfred\}@google.com}
}

\begin{document}

\maketitle

\begin{abstract}
    Expectation-Maximization (EM) is a prominent approach for parameter estimation of hidden (aka \mbox{latent}) variable models. Given the full batch of data, EM forms an upper-bound of the negative log-likelihood of the model at each iteration and updates to the minimizer of this upper-bound. We first provide a ``model level'' interpretation of the EM upper-bound as sum of relative entropy divergences to a set of singleton models, induced by the set of observations. Our alternative motivation unifies the ``observation level'' and the ``model level'' view of the EM. As a result, we formulate an online version of the EM algorithm by adding an analogous inertia term which corresponds to the relative entropy divergence to the old model. Our motivation is more widely applicable than the previous approaches and leads to simple online updates for mixture of exponential distributions, hidden Markov models, and the first known online update for Kalman filters. Additionally, the finite sample form of the inertia term lets us derive online updates when there is no closed-form solution. Finally, we extend the analysis to the distributed setting where we motivate a systematic way of combining multiple hidden variable models. Experimentally, we validate the results on synthetic as well as real-world datasets.

\end{abstract}

\vspace{-6mm}
\section{INTRODUCTION}
The goal of EM is to minimize\footnote{We adopt the
minimization view of the EM algorithm by considering the
negative of the log-likelihood function. This will simplify
our online EM motivation in the following.} the negative
log-likelihood (loss) of a hidden variable model given a set of iid observations from the data.
Instead of directly minimizing the negative log-likelihood,
EM forms an upper-bound of the loss at each iteration and
then updates to the minimizer\footnote{%
Or the approximate solution of the minimization problem.}
of the upper-bound. A basic lemma guarantees that due to
the tightness of the upper-bound at the current estimate of
the parameters, every EM update decreases the negative log-likelihood
(or keeps it unchanged when the current estimate is at a local minimum). 

We first provide an alternative view of the EM upper-bound as sum of joint relative entropy divergences between a set of ``singleton models'' and the new model. Each singleton model is induced by minimizing the Monte Carlo approximation of the relative entropy divergence to the data distribution, based on the single observation. This motivates the use of this divergence as an inertia term for an online variant of the EM algorithm.
We add the relative entropy divergence between the joint distributions at the old and the new model to the EM upper-bound and update to the minimizer.
Curiously enough, our divergence based online EM updates coincide with the
ones given in~\cite{sato} and their generalization in~\cite{cappe}\footnote{Which is based on a stochastic
approximation of the EM upper-bound.}. However, we will
show that the new formulation is more versatile
and gives online updates for more complex models.
In particular, our approach avoids having to identify the
sufficient statistics of the joint distribution. Additionally, it can handle cases where there exists no closed-form solution for the EM update using Monte Carlo approximation of the inertia term. Finally, we extend the same methodology to minimizing a sum of relative entropy divergences between different models in a distributed setting. This results in an efficient way of combining hidden variable models. To summarize:

\vspace{-0.3cm}
\begin{itemize}[itemsep=0.7mm, parsep=0pt]
    \item[--] We motivate the observation level view of the EM algorithm as combining models via minimizing sum of relative entropy divergences.
    \item[--] We formulate an online EM algorithm based minimizing the sum of divergences to the singleton models (observations) and the old model (inertia). The new formulation allows an approximate form of the online EM algorithm for cases where the updates do not have a closed-form.
    \item[--]  Using the new formulation, we obtain closed-form updates for mixtures of exponential
        distributions, hidden Markov models (HMMs), and Kalman filters, and approximate updates for the Compound Dirichlet distribution.
        \item[--] Most importantly, we develop divergences
	between hidden variable models and provide a method
	for combining such models by minimizing convex combinations of divergences.
\end{itemize}
\vspace{-0.4cm}
Here we only consider models
for which the EM upper-bound induces a closed-form.
We omit extensions where approximations are used for the upper bound
(e.g. variational inference \citep{blei}).

\paragraph{Previous work}
EM is one of the most well-studied algorithms due to its simplicity
and monotonic descent property~\cite{embook, do, gupta}.  It
was also shown that EM converges to a stationary point of the
negative log-likelihood under some mild conditions~\cite{wu}.
EM is naturally a batch algorithm.
Attempts for developing online versions of EM start with the work of~\citet{titterington},
who employs a second order method by approximating
the complete data Fisher information matrix.
This algorithm has been shown to almost surely converge to a local
minimum of the negative log-likelihood~\cite{wang}.
However, deriving the updated for this method requires calculating sophisticated derivatives and matrix inversions which makes it intractable for complex models such as HMMs and Kalman filters. We show that our alternative divergence based motivation of the EM algorithm reduces to the work of~\citep{cappe}, which substitutes the E-step by a stochastic approximation of the EM upper-bound while keeping the M-step unchanged. \citet{cappe} showed that for models where the complete data likelihood belongs to an exponential family, the updates correspond to stochastic approximation of sufficient statistics. While this is intuitive, identifying the complete data sufficient statistic for more complex models becomes infeasible in practice. On the other hand, our new formulation provides several advantages. First, we avoid characterizing the sufficient statistics by directly forming the inertia term between the current model and the updates. As a result, we can easily derive online EM updates for more complex models such as HMMs and Kalman filters. Additionally, we can apply the approximate form of the inertia term for problems where the minimization of the EM upper-bound does not have a closed-form solution.
Finally, the new divergences between hidden variable models
lead to a method for combining multiple hidden variable
models and this has useful applications in the distributed
setting.

Previous online EM algorithms for learning exponential family models have been mainly based on gradient ascent methods or heuristic approaches for maximizing the likelihood or updating the sufficient statistics.
Therefore, the resulting updates are commonly unstable and
require careful tuning of the parameters. Generally, these updates also lack performance guarantees. Specifically,
online methods have been developed for mixture of
exponential distributions \cite{neal, singer} and for online (aka block-wise) learning of HMMs
\cite{baldi,singer-hmm,cappe-quasi,mizuno}.
Also, inline (aka symbol-based)
methods have been proposed for learning
HMMs~\cite{krishnamurthy, collings, legland, garg, florez,
mongillo,cappe-hmm,kontorovich}. Our method falls
into the category of block-wise updates for HMMs. To the
best of our knowledge, no online algorithms were known for
Kalman filters. All earlier training methods were
based on either the batch EM algorithm via Kalman
smoothing or inline updates via Kalman
filtering~\cite{ghahramani}.

\section{BATCH EM MOTIVATION}
Given an iid sample $\Vis = \{\vis_n\}_{n=1}^N$
from an underlying data distribution $\Pdata(\vis)$, the EM algorithm
seeks to minimize the negative log-likelihood loss
    \[
        \nll{\Thett}{\Vis} = -\sfrac{1}{N}\sum_n \log \underbrace{\P(\vis_n|\, \Thett)}_{\int_{\hid} \P(\hid, \vis_n|\, \Thett)} \, , \vspace{-2mm}
    \]
wrt the parameters $\Thett$.
Here $\vis_n$ is the $n$-th observation of some
visible variable and $\hid$ denotes the hidden variable.
The above minimization problem, which involves logs of
integrals (or sums in the discrete case) is typically
non-convex and infeasible in practice.
Adding a divergence to a loss can simplify the
minimization.
Batch EM employs the following upper-bound of the loss:
\begin{align}
    \UP_{\Theta}(\Ts|\Vis)
    := & -\sfrac{1}{N}\sum_n \log \int_h \P(h, v_n|\Thett)
	\nonumber\\
	&+ \sfrac{1}{N}\sum_n\int_{h}\P(h|v_n,\Theta)
      \log\frac{\P(h|v_n,\Theta)}{\P(h|v_n,\Ts)}
\nonumber\\
    = &- \sfrac{1}{N}\sum_n \EE_{\P(h|\,v_n, \Theta)}\bigg[
              \log {\P(h, v_n|\,\Thett)}\bigg]
          \nonumber\\
          &- \sfrac{1}{N}\underbrace{\sum_n \mathbb{H}_{\scriptsize \Theta_n}\!(H|\,\vis_n)}_{\text{const.}}\,,\label{eq:em-up-simple}\vspace{-5mm}
\end{align}
where $\Theta$ denotes the current parameter set and $\mathbb{H}_{\scriptsize \Theta_n}\!(H|\,\vis_n)\coloneqq -\int_h \P(\hid|\,\vis_n, \Theta)\log \P(\hid|\,\vis_n, \Theta)$ is the conditional differential entropy of $H|\,\vis_n$.
Batch EM algorithm proceeds by forming the upper-bound by
calculating the posteriors $\P(h|v_n,\Theta)$ based on the
current estimate $\Theta$ (the E-step)
and then minimizing \eqref{eq:em-up-simple} wrt $\Thett$
and updating $\Theta$ to the minimized parameters (the M-step).
Minimizing the upper-bound is easier than minimizing the negative log-likelihood
directly because logs of integrals are now replaced by logs of joints. Since the upper-bound is tight, i.e. $\UP_{\Theta}(\Theta|\Vis) = \nll{\Theta}{\Vis}$, a decrement in the value of upper-bound amounts to a reduction in negative log-likelihood, that is, $\UP_{\Theta}(\Theta^{\text{\tiny new}}|\Vis) < \UP_{\Theta}(\Theta|\Vis) \Rightarrow  \nll{\Theta^{\text{\tiny new}}}{\Vis} <  \nll{\Theta}{\Vis}$.

We now rewrite the upper-bound as sum of relative entropy divergences to a set of singleton models. Given the current model estimate $\Theta$, let
\begin{equation}
    \label{eq:singleton}
    \P(\hid, \vis|\, \Theta_n) \coloneqq \delta_{\vis_n}(\vis)\, \P(\hid|\, \vis, \Theta)\, ,
\end{equation}
where $\delta_{\vis_n}(\vis)$ is the Dirac measure centered at $\vis_n$. Note that $\Theta_n$ is an estimate of the model that minimizes the upper-bound~\eqref{eq:em-up-simple} at $\Theta$ using a single observation $\vis_n$. Then, the relative entropy divergence between the models $\Theta_n$ and $\Thett$ becomes
\begingroup
\allowdisplaybreaks
\begin{align}
    \label{eq:re-singleton}
    &\Delta_{\text{\tiny RE}}(\Theta_n, \Thett) =
    \int_{h, v} \P(v,h|\,\Theta_n) \log\frac{\P(v,h|\,\Theta_n)}{\P(v,h|\,\Thett)}\nonumber\\
    &= -\int_{\hid, \vis} \P(\hid, \vis|\, \Theta_n) \log\P(h, v|\,\Thett) -\mathbb{H}_{\scriptsize \Theta_n}\!(H, V)\nonumber\\
    &= -\EE_{\P(h|\,v_n, \Theta)}\bigg[
              \log \P(h, v_n|\,\Thett)\bigg] -\underbrace{\mathbb{H}_{\scriptsize \Theta_n}\!(H, V)}_{\text{const.}}
          .\vspace{-5mm}
\end{align}
\endgroup
where $\mathbb{H}_{\scriptsize \Theta_n}\!(H, V) \coloneqq -\int_{\hid, \vis} \P(\hid, \vis|\, \Theta_n)\log\P(\hid, \vis|\, \Theta_n)$ is the joint differential entropy of $H$ and $V$. Note that $\mathbb{H}_{\scriptsize \Theta_n}\!(H, V) = \mathbb{H}_{\scriptsize \Theta_n}\!(V) + \mathbb{H}_{\scriptsize \Theta_n}\!(H|\, V)$.
Thus, the $n$-th term in the EM upper-bound~\eqref{eq:em-up-simple} has the exact same form as~\eqref{eq:re-singleton} minus the $\mathbb{H}_{\scriptsize \Theta_n}\!(V)$ term. Although this term is unbounded for the singleton distribution in~\eqref{eq:singleton}, it acts a constant wrt $\Thett$ and can be omitted from the upper-bound. Thus, one step of the EM algorithm can be seen as minimizing sum of the relative entropy divergences~\eqref{eq:re-singleton} to the singleton distributions~\eqref{eq:singleton}, that is,
\[
    \UP_{\Theta}(\Ts|\Vis) =  \sfrac{1}{N}\,\sum_n \Delta_{\text{\tiny RE}}(\Theta_n, \Thett) + \sfrac{1}{N}\underbrace{\sum_n \mathbb{H}_{\scriptsize \Theta_n}\!(V)}_{\text{const.}}\, .
\]

\section{ONLINE EM MOTIVATION}
\label{sec:new}

EM is naturally a batch algorithm and requires the full set of observations to carry out each iteration. On the other hand, online algorithms only receive one example (or a small
batch $\Vis^t$) at every iteration $t$.
The updates minimize the loss of the given batch
(in our case an upper-bound of the loss) plus an \emph{inertia} term
(a second divergence)
that ensures that the updates remain close to the current estimates $\Theta^t$. By the model view of the EM upper-bound in~\eqref{eq:re-singleton}, it is natural to choose the inertia term in the same form, i.e. a relative entropy divergence to the current model $\Theta^{(t)}$. Thus, the online update minimizes
\begin{equation}
\label{eq:online-em}
    \Theta^{(t+1)}\! =\! \underset{\Thett}{\argmin}    \underbrace{\UP_{\Theta^t}(\Thett|\,\Vis^{(t)})}_{\text{loss}}	+\sfrac{1}{\eta^{(t)}}\, \underbrace{\Delta_{\text{\tiny RE}}(\Theta^{(t)}, \Thett)}_{\text{inertia}} ,\vspace{-2mm}
\end{equation}
where $\Theta^{(t)}$ and $\Vis^{(t)}$ are the parameters and the given batch of observations at round $t$, respectively. Based on our discussion in the previous section, update~\eqref{eq:online-em} corresponds to combining $\vert\Vis^{(t)}\vert + 1$ models, therefore is guaranteed to have the same form as a batch EM update. The parameter $\eta$ can be seen as a learning rate which controls the extent that the parameters are affected by the new observations. Note that $\eta \rightarrow \infty$ recovers the vanilla EM algorithm on the batch $\Vis^{(t)}$ while $\eta \rightarrow 0$ keeps the parameters unchanged. Moreover, following the tightness of the upper-bound, the objective of~\eqref{eq:online-em} is equal to $\nll{\Theta^{(t)}}{\Vis^{(t)}}$ at $\Thett = \Theta^{(t)}$. Thus, every step of the online EM algorithm decreases the negative log-likelihood of the model over the batch $\Vis^{(t)}$.

A few remarks are in order. The objective~\eqref{eq:online-em} is equal to the objective function of the online EM algorithm of~\cite{cappe} up to additive constant terms wrt $\Thett$.

\begin{theorem}
    \label{obs:cappe}
    The objective function in~\eqref{eq:online-em} is equal to the objective function of the online EM algorithm of~\cite{cappe} up to additive constant terms wrt $\Thett$.
\end{theorem}
The proof is given in the appendix.

The formulation in~\citep{cappe} is specifically applied to models where the complete data likelihood belongs to an exponential family and the updates are shown to reduce to stochastic approximation of the sufficient statistics. While their approach is applicable to simpler models such as mixture of Poisson, identifying the sufficient statistics immediately becomes tedious for more complex models such as HMMs and Kalman filters. As a result, the corresponding updates for these models had not been discovered. Moreover, the decrement of the negative log-likelihood over the current observation $\Vis^{(t)}$ is not evident in this formulation.
\begin{corollary}
    Under mild assumptions on the parameter space of the exponential family model and using a decaying learning rate $0 < \eta^{(t)} < 1, \sum_{t=1}^\infty  \eta^{(t)} = \infty$ and $ \sum_{t=1}^\infty  \eta^{(t)} < \infty$, the update~\eqref{eq:online-em} almost surely converges to a stationary point that maximizes the expected log-likelihood w.r.t. the data distribution.
\end{corollary}

\subsection{NATURAL GRADIENT APPROXIMATION}
For small $\mathrm{d}\Thett \coloneqq \Thett - \Theta^{(t)}$
the inertia term can be approximated as
\begin{equation}
    \label{eq:fisher}
    \Delta_{\text{\tiny RE}}(\Theta^{(t)},\Thett) \approx\sfrac{1}{2}\, \mathrm{d}\Thett^\top\,I_{\text{F}}(\Theta^{(t)})\, \mathrm{d}\Thett \, ,
\end{equation}
where $I_{\text{F}}(\Theta^{(t)}) = -\EE_{\P(\hid, \vis,\Theta^{(t)})} \big[\nabla_{\Theta}^2\log \P(\hid,\vis|\,\Theta^{(t)})\big]$ is the Fisher information matrix. Using~\eqref{eq:fisher}, the update~\eqref{eq:online-em} can be approximated as
\[
    \Theta^{(t+1)} \approx \Theta^{(t)} - I^{-1}_{\text{F}}(\Theta^{(t)})\,\nabla\UP_{\Theta^{(t)}}(\Theta^{(t)}|\,\Vis^{(t)})\, ,
\]
which is called the natural gradient update~\cite{amari} and resembles the gradient-based updates for EM proposed in~\cite{embook} where the observed Fisher $I_{\text{O}}(\Theta^{(t)}) = -\EE_{\P(\vis|\,\Theta^{(t)})} \big[\nabla_{\Theta}^2\log \P(\vis|\,\Theta^{(t)})\big]$ is used in place of $I_{\text{F}}(\Theta^{(t)})$. Using the equality,
\begin{align*}
    \nabla&\UP_{\Theta^{(t)}}(\Theta^{(t)}|\,\Vis^{(t)})\\
    & = -\sfrac{1}{N}\,\sum_n \EE_{\P(h|\,v_n, \Theta^{(t)})}\bigg[
        \frac{\nabla \P(h, v_n|\,\Theta^{(t)})}{\P(h, v_n|\,\Theta^{(t)})}\bigg]\\
    & = -\sfrac{1}{N}\, \sum_n \int_{\hid} \frac{\nabla \P(h, v_n|\,\Theta^{(t)})}{\P(v_n|\,\Theta^{(t)})} = \nabla \nll{\Vis^{(t)}}{\Theta^{(t)}}\, ,
\end{align*}
we have,
\[
    \Theta^{(t+1)} \approx \Theta^{(t)} - I^{-1}_{\text{F}}(\Theta^{(t)})\,\nabla \nll{\Vis^{(t)}}{\Theta^{(t)}}\, .
\]
This connection was also observed in~\cite{sato} and the extension~\cite{cappe}.

\subsection{FINITE-SAMPLE APPROXIMATION}
In cases where $\Delta_{\text{\tiny RE}}(\Theta^{(t)}, \Thett)$ does not yield a closed-form solution, an approximate inertia term can be obtained via a finite number of samples as
\begin{align}
    &\Delta_{\text{\tiny RE}}(\Theta^{(t)},\Thett)\nonumber\\
    &\,\,= \EE_{\P(\vis|\,\Theta^{(t)})}\bigg[\EE_{\P(\hid|\,\vis,\Theta^{(t)})}\Big[\log\frac{\P(\hid,\vis|\Theta^{(t)})}{\P(\hid, \vis|\Thett)}\Big]\bigg]\\
    &\,\, \approx -\underbrace{\sfrac{1}{N^\prime}
    \sum_{n^\prime}\EE_{\P(h|\,v_{n^\prime},\Thett)}
        \bigg[\!\log\P(v_{n^\prime},h|\Thett)\bigg]}_{\UP_{\Theta^t}(\Thett|\, \Vis^\prime) + \text{const.}} +\, \text{const.}\nonumber
\end{align}
where the samples $\Vis^\prime = \{v_{n^\prime}\}_{n^\prime = 1}^{N^\prime}$
are drawn from the distribution~$\P(v|\,\Theta^{(t)})$, not the data distribution $\Pdata(\vis)$. Thus, the samples $\Vis^\prime$ may be seen as pseudo-observations
that encourage the solution to remain close to the current model parameters $\Theta^{(t)}$. We will use this sampled form of the inertia term to derive update for the compound Dirichlet model. The sampled form is also similar to the update given in~\citep{neal}. However, in their formulation, $\Vis^\prime$ is replaced with $\Vis - \{\vis_i\}$ where $\vis_i$ is a random sample from $\Vis$ and the upper-bound is formed at $\Theta^{(t-1)}$ instead of $\Theta^{(t)}$.

\section{UPDATES WITH CLOSED-FORM} \label{sec:examples}
The objective is easier to minimize when it reduces to a linear
combination of negative log-likelihoods of exponential
family distributions which includes mixtures of exponential
families, HMMs and Kalman Filters. Note that the latter two cases
are already hard to handle with the methodology of
\cite{cappe}. We start with some background on this family of
distributions. The exponential family~\cite{wainwright}
with vector of sufficient statistics $\phi(\x)$ and natural parameter $\thetb$ is defined as
$\P_G(\x|\thetb) = \exp( \thetb\cdot \phi(\x) - G(\thetb)), $
where $G(\thetb) = \log \int_{\x} \exp\big(\thetb\cdot \phi(\x)\big)$
is called the log partition function that ensures that $\P_G(\x|\,\thetb)$ integrates to one.
The expectation parameter
$\mub = g(\thetb) = \int_{\x} \phi(\x)\, \P_G(\x|\,\thetb)$
is the dual~\cite{urruty} of the natural parameter $\thetb$ where $g(\thetb) \defeq \nabla G(\thetb)$.
The duality implies $ \thetb = g^*(\mub) = g^{-1}(\mub)$ where $g^*(\mub) \coloneqq \nabla G^*(\mub)$ and $G^*(\mub) = \sup_{\thetb^\prime}\{\mub\cdot\thetb^\prime - G(\thetb^\prime)\}$ is the convex conjugate of $G$.
It is easy to show that $G(\thetb)$ is a convex function. In fact, the relative entropy divergence between two exponential distributions (of the same form) with parameters $\thetb$ and $\thetbt$ yields
\begin{align*}
    \int_{\x} \P_G(\x|\,\thetb) \log
    \frac{\P_G(\x|\,\thetb)}{\P_G(\x|\,\thetbt)} =
    \Delta_G(\thetbt,\thetb) = \Delta_{G^*}(\mub, \mut)\, ,
\end{align*}
where 
$
\Delta_G(\ts, \theta) =  G(\ts) - G(\theta) -
g(\theta)\cdot (\ts - \theta)\, ,
$
is the \emph{Bregman divergence}~\cite{bregman} induced
by the convex function $G(\cdot)$ and $\frac{\partial}{\partial \ts}\, \Delta_G(\ts, \theta) = g(\ts) - g(\thetb)$.
The following lemma will be useful for deriving the updates.
    \begin{lemma}
    \label{lem:partial-comb}
    For $\{\alpha_m\}_{m=1}^M$ s.t. $\alpha_m \in \RR_+$ and $\sum_m \alpha_m > 0$,
        \begin{align*} \theta_{opt} & = \argmin_{\ts} \sum_m \alpha_m \big(G(\ts)- \ts \cdot \mu_m \big)\\
            & = g^{-1}{\scriptstyle \Big(\frac{\sum_m \alpha_m\, \mu_m}{\sum_m \alpha_m}\Big),}
    \end{align*}
    i.e. $ \mu_{opt} = \frac{\sum_m \alpha_m\,
	\mu_m}{\sum_m \alpha_m}.$
    \end{lemma}

\subsection{MIXTURE OF EXPONENTIAL FAMILY}
In $k$-mixture of exponential, model $h \in [k]$ is chosen according to the probability $\w_h \defeq \P(h|\,\Theta)$ and the observation is drawn from the corresponding distribution $P(v|\, h, \Theta) =  \P_G(v|\,\thetb_h) = \exp(\thetb_h \cdot \phi(v) - G(\thetb_h))$, which belongs to an exponential family. Thus, the model parameters are $\Theta = \{\w_h, \mub_h(\thetb_h)\}_h$. The joint distribution becomes
$$
\P(v, h|\, \Theta) = \w_h\, \exp(\thetb_h \cdot \phi(v) - G(\thetb_h))\, ,
$$
while the marginal is simply a sum over all states,
$$
\P(v|\, \Theta) = \sum_h \w_h\, \exp(\thetb_h \cdot \phi(v) - G(\thetb_h))\, .
$$
The EM upper-bound can be formed using the posterior distributions of each observation $v_n$, that is,
\begin{align*}
    &\UP_{\Theta}(\Thett|\, \Vis)\\
    & =  -\sfrac{1}{N}\sum_n \sum_h \gamma_{n,h}\, \Big(\log \wt_h + \big(G(\thetbt_h) - \thetbt_h \cdot \phi(v_n)\big)\Big)\, ,
\end{align*}
where we ignored the constants. The posteriors $\gamma_{n,h}$ are calculated as
$$
\gamma_{n,h} = \frac{\w_h\, \exp(\thetb_h \cdot \phi(v_n) - G(\thetb_h))}{\sum_{h^\prime} \w_{h^\prime}\, \exp(\thetb_{h^\prime} \cdot \phi(v_n) - G(\thetb_{h^\prime}))}\, .
$$
The inertia term for the online EM algorithm becomes
\begingroup
\allowdisplaybreaks
 \begin{align*}
        &\Delta(\Theta, \Thett) = \sum_h \int_v \w_h\, \P_G(v|\,\thetb_h)\, \log\frac{\w_h\, \P_G(v|\,\thetb_h)}{\wt_h\, \P_G(v|\,\thetbt_h)}\\
        & = \sum_h \w_h \log\frac{\w_h}{\wt_h} + \sum_h \w_h\, \underbrace{\int_v \P_G(v|\thetb_h) \log \frac{\P_G(v|\thetb_h)}{\P_G(v|\thetbt_h)}}_{\Delta_G(\thetbt_h,\thetb_h)}\, .
    \end{align*}
\endgroup
Combining the inertia term with the upper-bound and applying Lemma~\ref{lem:partial-comb}, we have
\begin{align}
    \wt_h  & = \frac{\sfrac{1}{\eta}\, \w_h + \sfrac{1}{N}\,\sum_n
    \gamma_{n,h}}{\sfrac{1}{\eta} + 1}\, ,\label{eq:our-w}\\
        \mut_h & = \frac{\sfrac{1}{\eta}\, \w_h\, \mub_h +
    \sfrac{1}{N}\,\sum_n \gamma_{n,h}\, \phi(v_n)}{\sfrac{1}{\eta}\, \w_h + \sfrac{1}{N}\,\sum_n \gamma_{n,h}}\, .\label{{eq:our-mu}}
    \end{align}

\subsection{HIDDEN MARKOV MODELS}
A Hidden Markov Model (HMM)~\cite{rabiner} consists of an underlying finite-state (hidden) Markov chain with probability of transitioning from state $h$ to $h^\prime$ equal to $a_{h,h^\prime} \defeq P(h|\, h^\prime, \Theta)$ and an initial state probability equal to $\pi_{h_1} \defeq P(h_1|\, \Theta)$. At every round, the model makes a transition to a new state according to the state transition probabilities and given the new state $h$, generates an observation according to the state emission probability $P(v|\, h, \Theta)$. We make the assumption that the state emission probabilities are members of an exponential family, that is, $P(v|\, h, \Theta) = P_G(v|\, \thetb_h)$. Thus, the model parameters are $\Theta = \{\pi_h, \{a_{h,h^\prime}\}_{h^\prime}, \mub_h(\thetb_h)\}_h$. The joint distribution of the model can be written as
$$
P(v, h|\, \Theta) = \prod_{t=1}^T a_{h_{t-1},h_t}\, P_G(v_t|\, \thetb_{h_t})\, ,
$$
in which, we define $a_{h_0,h_1} \defeq \pi_{h_1}$ and the marginal can be obtained by summing over all the possible hidden states
$$
P(v|\, \Theta) = \sum_{h_1,\ldots,h_T}\, \prod_{t=1}^T a_{h_{t-1},h_t}\, P_G(v_t|\, \thetb_{h_t})\, .
$$
The EM upper-bound can be written as
\begin{align*}
    \UP_{\Theta}(\Thett|\, \Vis) & = \sfrac{1}{N}\,\sum_n\Big(\sum_h\gamma^{n,1}_{h}\log\frac{\gamma^n_{h_1}}{\pit_{h_1}}\\
    & + \sum_{t=1}^{T-1}\sum_{h,h^\prime}  \gamma^{n,t}_{h,h^\prime} \log\frac{\gamma^{n,t}_{h,h^\prime}}{\at_{hh^\prime}}\\
    & + \sum_{t=1}^T \sum_h \gamma^{n,t}_{h}\, \big(G(\thetbt_{h_t}) - \phi(v_{nt})\cdot \thetbt_{h_t}\big) \Big)\, ,
\end{align*}
in which, we again ignore the constant terms. The state posteriors are found using the Baum-Welch algorithm by performing a forward-backward pass~\cite{rabiner}. We define
\begin{align*}
    \gamma^{n,t}_{h} &\defeq P(h_{n,t} = h|\, v_n, \Theta)\, ,\\
    \gamma^{n,t}_{h,h^\prime} &\defeq P(h_{n,t+1} = h^\prime, h_{n,t} = h|\, v_n, \Theta)\, .
\end{align*}
The inertia term can be written as
 \begin{align*}
     \Delta(\Theta, \Thett) & = \sum_{h}
	\pi_{h}\log\frac{\pi_{h}}{\pit_{h}} +
	\sum_h u_h \sum_{h'} a_{h,h'}
	\log\frac{a_{h,h'}}{\at_{h,h'}}\\
     & + \sum_h
	u_h\,\Delta_{G}(\ts_h, \thetb_h)\, ,
        \end{align*}
        where
        $
        \, u_h = \sum_{t=1}^\infty \delta^t_h\, ,
        $
        with\,
        $
        \delta^1_h = \pi_h\,\, \text{, and }\,\, \delta^{t+1}_{h^\prime} = \sum_{h}\delta_h^t\,a_{h,h^\prime}\, .
        $
        In other words, $u_h$ is the expected usage of state $h$. Note that the usage $u_h$ is not finite in general and should be instead approximated by a finite length sequence. However, for the class of \emph{absorbing} HMMs, we can calculate the usages in the exact form. More specifically, the transition matrix $A = [a_{h,h^\prime}]$ of an absorbing HMM with $r$ absorbing states has the following form
        $$
        A = \left[\begin{array}{cc}
                Q & R\\
                0 & I_r\\
        \end{array}
                \right]\, ,
        $$
        where the $Q$ entails the transition probabilities from a transient state to another while $R$ describes the transition probabilities of from transient states to absorbing states. $I_r$ is an identity matrix which describes the transitions from each absorbing state back to itself. The expected usages of the transient states can be calculated as
        $$
        u^\top = \pi^\top + \pi^\top Q + \pi^\top Q^2 + \ldots = \pi^\top (I - Q)^{-1}\, .
        $$
        Additionally, note that for an absorbing state $h$, we always have $\sum_{h^\prime} a_{h,h^\prime}\log\frac{a_{h,h^\prime}}{\at_{h,h^\prime}} = 0$. Therefore, the corresponding terms can be omitted from the inertia term.

Combining the EM upper-bound and the inertia term and applying Lemma~\ref{lem:partial-comb} gives the following updates
\begin{align*}
    \pit_h &= \frac{\sfrac{1}{\eta}\, \pi_h + \sfrac{1}{N}\,\sum_n
    \gamma_h^{n,1}}{\sfrac{1}{\eta}\, + 1}\, ,\\
    \at_{h,h^\prime} & = \frac{\sfrac{1}{\eta}\, u_h\,
    a_{h,h^\prime}+ \sfrac{1}{N}\,\sum_n \sum_t
    \gamma_{h,h^\prime}^{n,t}}{\sfrac{1}{\eta}\, u_h + \sfrac{1}{N}\,\sum_n \sum_t \gamma_h^{n,t}}\, ,\\
    \mut_h &= \frac{\sfrac{1}{\eta}\, u_h\, \mub_h + \sfrac{1}{N}\,\sum_n
    \sum_t \gamma_h^{n,t}\, \phi(v_{n,t})}{\sfrac{1}{\eta}\, u_h + \sfrac{1}{N}\,\sum_n \sum_t \gamma_h^{n,t}}\, .
\end{align*}


\subsection{KALMAN FILTERS}
Kalman filters~\cite{welch} can be described using the following two update equations
       \begin{align*}
        h_{t+1} & = A\, h_t + \rho_t\, ,\\
        v_t & = C\, h_t + \veps_t\, ,
        \end{align*}
where $h_t$ is the underlying (hidden) state at $t$ and $v_t$ is the corresponding output. Both state and observation noise, $\rho_t$ and $\veps_t$, are zero-mean Gaussian random variables with covariance matrices equal to $Q$ and $R$, respectively. The initial state $h_1$ is generally assumed to be drawn from a Gaussian distribution with mean $\pi_1$ and covariance $V$. Thus, the model parameters are $\Theta = \{\pi_1, V, A, C, Q, R\}$. In Kalman filters, only the output is observed and thus, the state as well as the noise variables are hidden.

The joint distribution of a Kalman filter can be written as
$$
\P(h,v|\, \Theta) = \prod_{t=1}^T  \P(h_t|\, h_{t-1}, \Theta)\,  \P(v_t|\, h_t, \Theta)\, ,
$$
where
\vspace{-7mm}
 \begin{align*}
     \P(h_1|\, h_0, \Theta) &\coloneqq \P(h_1|\, \Theta)\\
     \P(v_t|\,h_t,\Theta) & \sim \normal{C\, h_t}{R}\,
     ,\\
           \P(h_{t+1}|\,h_t,\Theta) & \sim \normal{A\, h_t}{Q}\, ,\\
         \P(h_1|\,\Theta) & \sim \normal{\pi_1}{V}\, .
    \end{align*}
Here, $\normal{\xi}{\Sigma}$ denotes a Gaussian probability density with mean and covariance equal to $\xi$ and $\Sigma$, respectively. The marginal can be obtained by integrating over all the state variables, that is,
$$
P(v|\, \Theta) = \int_{h_1,\ldots,h_T}\, P(h,v|\, \Theta) \, .
$$
Forming the EM upper-bound requires calculating the posteriors. Note that because all the random variables are Gaussian, it suffices to keep track of the means and covariances. The posteriors depend on the following three expectations
$$  \xh_t  \defeq \EE[h_t|v],\; P_t \defeq \EE[h_t
    h_t^\top| v],\;
    P_{t,t-1}  \defeq \EE[h_t h_{t-1}^\top| v],$$
which can be calculated recursively using the Kalman filtering and Kalman smoothing equations~\cite{ghahramani}. Thus, the EM upper-bound can be written as
\begin{align*}
    2\times \UP_{\Theta}(\Thett|\,\Vis) & = \sfrac{1}{N}\, \sum_n \bigg(\tr(\Vt^{-1} \Vb_n) + \log\vert\Vt\vert\\
    & + \sum_{t=2}^T\, \tr(\Qt^{-1}
    \Qb_{n,t}) + (T-1)\,\log\vert\Qt\vert\\
    & + \sum_{t=1}^T \tr(\Rt^{-1} \Rb_{n,t}) + T\,\log\vert\Rt\vert\bigg)\, ,
    \end{align*}
    where\footnote{Note that $n \in [N]$ denotes the observation index.}
    \begin{align*}
        \Vb_n & = P^n_1 - \pit_1 \xh_{n,1}^\top -  \xh_{n,1}\pit_1^\top + \pit_1\pit_1^\top\, ,\\
    \Qb_{n,t} & = P^n_t - \At P^n_{t-1,t} - P^n_{t,t-1}\At^\top + \At P^n_{t-1} \At^\top\, ,\\
        \Rb_{n,t} & = v_{n,t} v_{n,t}^\top - \Ct \xh_{n,t} v_t^\top - v_{n,t}\xh_{n,t}^\top \Ct^\top + \Ct P_t \Ct^\top\, .
\end{align*}
Again, our inertia term for the online EM algorithm is relative entropy between the joints, assuming a fixed observation length equal to $T$, that is,
\begin{align*}
    & 2\times \Delta(\Theta, \Thett)\\
    & = \tr\big(\Vt^{-1}(\pi_1 - \pit_1)(\pi_1 - \pit_1)^\top\big) + \, D_{ld}(V, \Vt)\\
    & + \tr\big(\Qt^{-1} (A\!-\!\At)\!
    \sum_{t=1}^{T-1} U_t (A\! -\! \At)^\top\big) +
    (T\!-\!1)\, D_{ld}(Q,\! \Qt)\\
    &  + \,\tr\big(\Rt^{-1} (C\! -\! \Ct)\!
    \sum_{t=1}^{T} U_t (C\! -\! \Ct)^\top \big) +
    T\, D_{ld}(R,\! \Rt)\, ,
\end{align*}
where $U_1 = V + \pi_1\pi_1^\top$ and $U_{t+1} =  Q + A\, U_t\, A^\top.$
Moreover $ D_{ld}(X,Y) = \tr(X Y^{-1}) - \log\vert X Y^{-1}\vert - d,$
    is the log-determinant divergence~\cite{nmf}.

    Combining the inertia term with the EM upper-bound and setting the derivatives wrt the parameters to zero yields
\begingroup
\allowdisplaybreaks
\begin{align*}
    \pit_1 & = \frac{\sfrac{1}{\eta}\, \pi_1 + \sfrac{1}{N}\, \sum_n \xh_{n,1}}{\sfrac{1}{\eta} + 1}\, ,\\
    \Vt & = \frac{\sfrac{1}{\eta}\big(V + (\pi_1 - \pit_1)(\pi_1 - \pit_1)^T)\big) + \Vb}{\sfrac{1}{\eta} + 1}\, ,\\
    \At & = \Big(\sfrac{1}{\eta} \sum_{t=2}^{T-1} A\, U_t + \sum_{t=2}^T P_{t,t-1}\Big)\, S_{T-1}^{-1}\, ,\\
    \Ct & = \Big(\sfrac{1}{\eta} \sum_{t=1}^T C\, U_t + \sfrac{1}{N}\,\sum_n\sum_{t=1}^T v_{n,t}\,\xh_{n,t}^\top\Big)\, S_T^{-1}\, ,\\
           \Qt & = \frac{\sfrac{1}{\eta}\big(Q + \Delta A_U \big) + \sfrac{1}{\big(N\,(T-1)\big)}\sum_n\sum_{t=2}^T \Qb_{n,t}}{\sfrac{1}{\eta} + 1}\, ,\\
     \Rt & = \frac{\sfrac{1}{\eta}\big(R + \Delta C_U \big)
    + \sfrac{1}{(N\,T)}\, \sum_n\sum_{t=1}^T \Rb_{n,t}}{\sfrac{1}{\eta} + 1}\, ,\
    \end{align*}

    \vspace{2mm}
     where
     \vspace{-10.5 mm}
     \begin{align*}
         S_T & = \sum_{t=1}^T \big(\sfrac{1}{\eta}\, U_t + P_t \big)\, ,\\
  \Delta A_U & = \sfrac{1}{(T-1)}\,\sum_{t=2}^{T} (A-\At)\, U_t\, (A - \At)^\top\, ,\\
    \Delta C_U &=  \sfrac{1}{T}\,\sum_{t=1}^{T} (C-\Ct)\, U_t\, (C - \Ct)^\top.
    \end{align*}
\endgroup

\begin{remark}
\label{obs:nll}
For mixtures of exponential families, HMMs, and Kalman
    filters the following holds:
    for a sufficiently small learning rate $\eta$,
the negative log-likelihood wrt the underlying unknown distribution $P_{\text{\tiny UK}}(v)$,
$$
    \EE_{\Pdata}\Big[\log P(v|\,\Theta)\Big] = \int_v P_{\text{\tiny UK}}(v)\, \log \int_h P(h, v|\,\Theta)\, ,
$$
improves after each update.
\end{remark}
This is a direct result of applying Proposition 3 in~\citep{cappe}. See~\citep{cappe} for further details.

\section{UPDATES WITHOUT CLOSED-FORM}

The minimization problem of the M-step of batch EM does not
always have a closed-form solution.
In those cases, it is likely the divergence term between
the models also does not have a closed-form either.
An example of such a model is the compound Dirichlet distribution~\cite{gupta}.
In this case, applying the online EM updates in form of~\eqref{eq:online-em} is infeasible.
However, we can use the finite sample form of the inertia term.
That is, in each iteration, we draw $N^\prime$ samples $\Vis^\prime = \{v_{n^\prime}\}_{n^\prime = 1}^{N^\prime}$ from $P(v|\, \Theta^t)$ and form the corresponding EM upper-bound by treating the additional samples as pseudo-observations. The update can be achieved by (numerically) minimizing the combined upper-bounds,
\begin{equation*}
\label{eq:online-em-app}
\Theta^{t+1} \approx \argmin_{\Thett}\,
    \UP_{\Theta^t}(\Thett|\,\Vis^t)	+\sfrac{1}{\eta}\, \UP_{\Theta^t}(\Thett|\, \Vis^\prime)
 \, ,
\end{equation*}
where again $\eta > 0$ is a learning rate parameter.
Note that this is fundamentally different than combining the samples $\Vis^t \cup \Vis^\prime$ and forming a single upper-bound. In fact, combining the samples may require larger number of pseudo-samples as we proceed with the online updates (which corresponds to a decaying learning rate) while our approach can be carried out by a fixed number of samples at every iteration.
Note that the quality of the approximation of the inertia term depends
on the size of the pseudo sample $N^\prime$.
Approximating the inertia term leads to a higher variance.
This variance decreases with $N^\prime$ but comes at a cost.
We will show an experimental result on the online parameter estimation of compound Dirichlet distribution.

\begin{figure*}[t!]
\vspace{-0.5cm}
\begin{center}
    \subfigure[]{\includegraphics[width=0.32\textwidth]{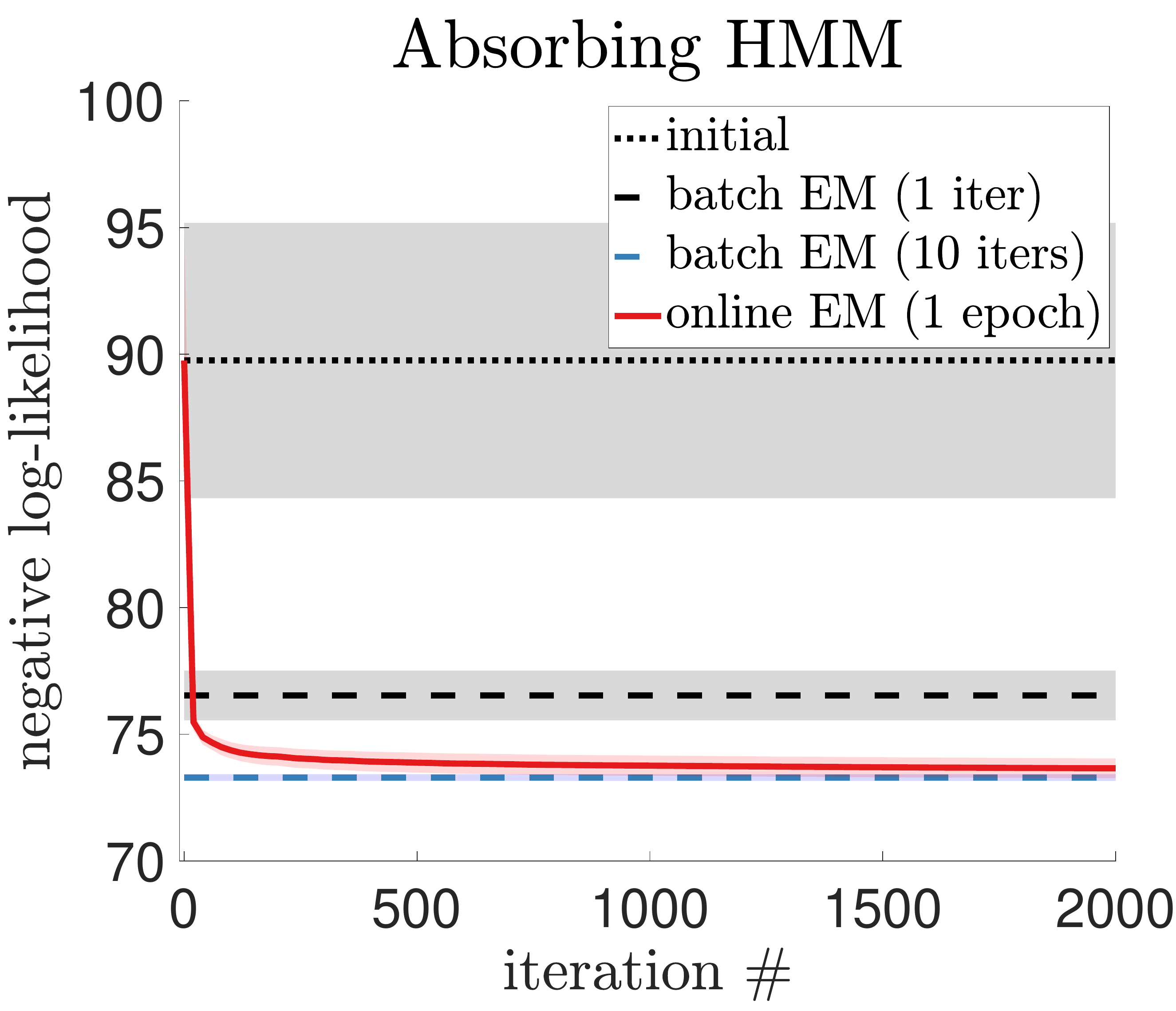}}\,
    \subfigure[]{\includegraphics[width=0.32\textwidth]{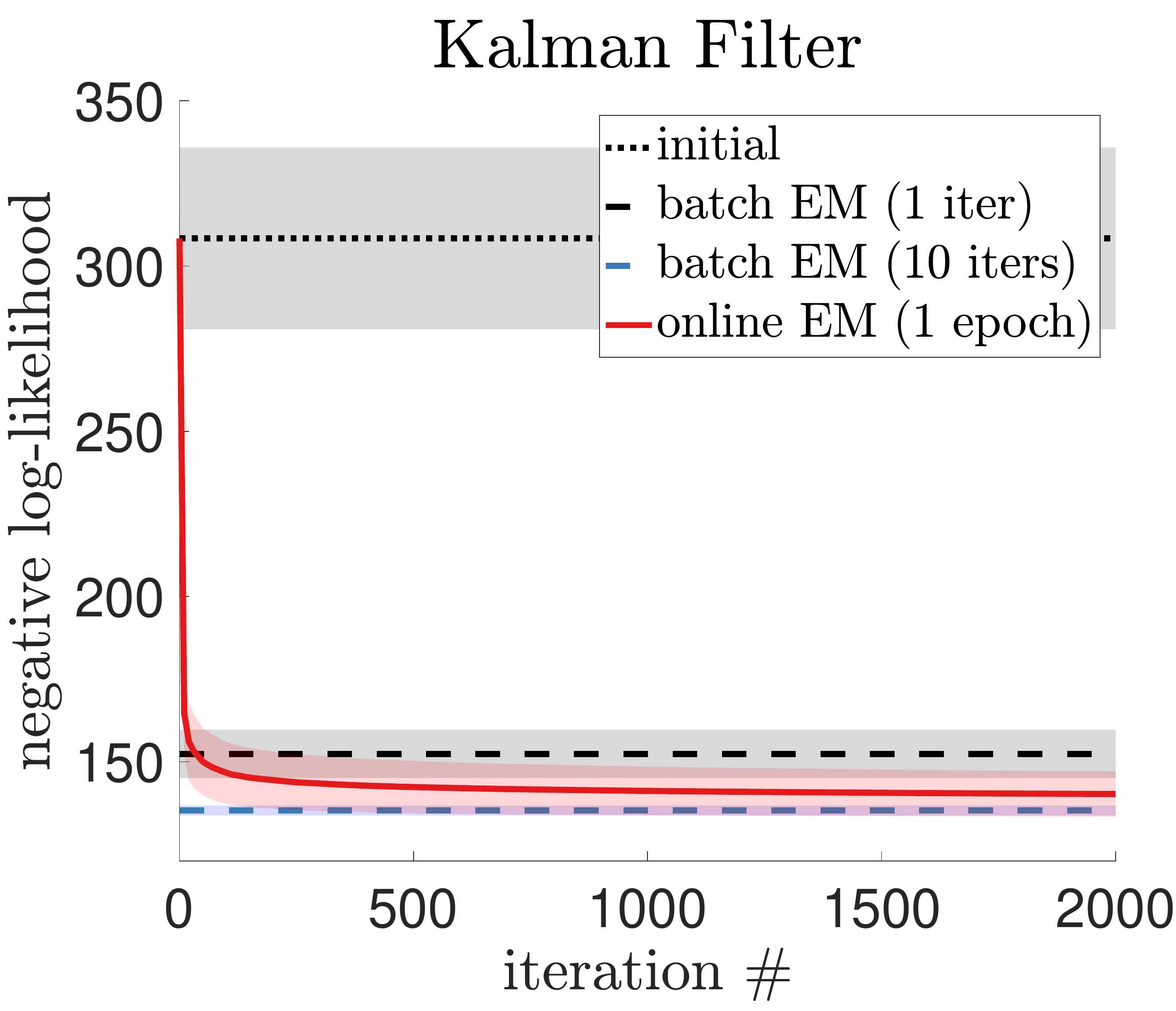}}\,
    \subfigure[]{\includegraphics[width=0.32\textwidth]{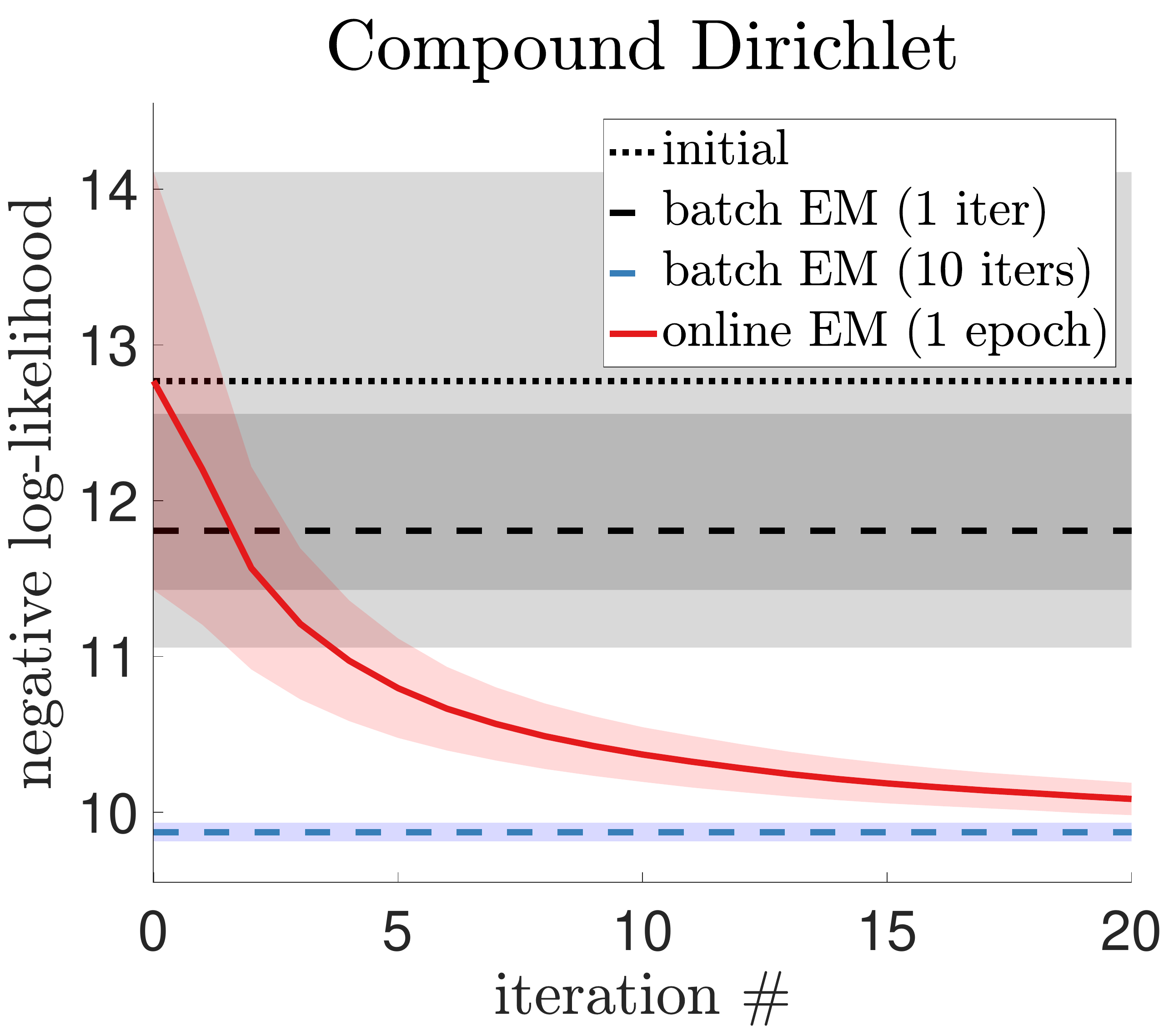}}\,
    \vspace{-0.45cm}
     \caption{Online EM results: a) absorbing HMM, b) Kalman filter, and c) compound Dirichlet distribution.}\label{fig:results}
\end{center}
    \vspace{-0.75cm}
\end{figure*}

\section{COMBINING MODELS}
\label{s:combo}

In many cases, multiple \emph{local models} need to be combined to form a \emph{global model}. For instance due to the large amount of data, the model training is distributed over multiple machines where each machine only receives a subset of the dataset and performs updates on its local model. The local models are then combined into a single global model at the end of each iteration (synchronous) or the end of the training process (asynchronous). Our divergences between the hidden variable models provide a natural way of combining the local models in a distributed setting.
More formally, let $\Theta^{(m)}$ denote the set of local parameters of model
$m \in [M]$. We can define the combined model parameters $\Theta^{(\text{comb})}$ as
\begin{equation}
    \label{eq:comb}
    \Theta^{(\text{comb})} = \argmin_{\Thett} \sum_{m\in [M]}\, \alpha_m\, \Delta_{\text{\tiny RE}}\big(\Theta^{(m)}, \Thett\big)\, ,
\end{equation}
where $\alpha_m \geq 0$ is the associated weight for combining model $m$ (s.t. $\sum_m \alpha_m > 0$. The value of $\alpha_m$ can be tuned based on the amount of data seen by model $m$, accuracy of the solver, etc. For the exponential family models, updates in~\eqref{eq:comb} reduces to a convex combination of the complete data sufficient statistics of the models\footnote{And not the sufficient statistics of the components.}. As an example, for hidden Markov models, Equation~\eqref{eq:comb} yields
\begin{align*}
 \pi^{(\text{comb})}_h &= \frac{\sum_m \alpha_m \pi^{(m)}_h}{\sum_m \alpha_m}\, ,\\
    a^{(\text{comb})}_{h,h^\prime} & = \frac{\sum_m \alpha_m u^{(m)}_h\,
    a^{(m)}_{h,h^\prime}}{\sum_m \alpha_m u^{(m)}_h}\, ,\\
    \mu^{(\text{comb})}_h &= \frac{\sum_m  \alpha_m u^{(m)}_h\, \mu^{(m)}_h}{\sum_m \alpha_m u^{(m)}_h}\, .
\end{align*}
We experimentally show that combining the models
via~\eqref{eq:comb} provides improved results compared to
the commonly used methods of combining the models via
other ways of averaging~\cite{open}. 

Again for cases where the divergence between the model does not admit a closed-form, we can use the sampling form of the divergence to combine the models. That is, we draw $N^\prime_m$samples $\Vis^\prime_m$ from $P(v|\, \Theta^{(m)})$ and form $\UP_{\Theta^{(m)}}(\Thett|\, \Vis^\prime_m)$. The combined model can be obtained as
$$
\Theta^{(\text{comb})} = \argmin_{\Thett} \sum_m\, \alpha_m\, \UP_{\Theta^{(m)}}(\Thett|\, \Vis^\prime_m)\, .
$$


\section{EXPERIMENTS}

In this section, we first conduct experiments on online
learning of absorbing HMMs and Kalman filters. Next, we
apply the approximate form of the inertia term for
estimating a compound Dirichlet distribution in which the
updates (as well as the inertia term) do not have a closed
form solution. Finally, we consider learning of
Gaussian mixture models in a distributed setting. In all experiments, we use a decaying learning rate of the form $\eta = \eta_0/t^\beta$ where $t$ is the iteration number and $\eta_0 > 0$ and $0.5 < \beta < 1$ are specified for each case. We repeat each experiment over $20$ random initializations.

\subsection{ABSORBING HMM}
We validate the derived updates for HMMs by conducting an experiment on estimating the parameters of an absorbing HMM with $3$ transient and a single absorbing state ($4$ hidden states in total) and Gaussian emission probabilities of dimension $4$. We consider \num{2000} samples from the model and apply batch EM updates as well as online updates with $(\eta_0, \beta) = (0.5, 0.9)$. The results are shown in Figure~\ref{fig:results}-a.
The online algorithm processes one observation per iteration.
Regarding processing time, one pass of the online update
over the entire data set (called one epoch) is comparable to a single batch
update. The online update rapidly outperforms the single batch EM update after
around $30$ iterations, and at the end of the first epoch converges to a value close
to the loss of $10$ batch EM iterations. Also the online updates are stable
to using lower or higher learning rates: The final loss
values obtained for $\eta_0 = 0.1$ and $\eta_0 = 1.0$ are
$77.17$ and $72.27$, respectively (not shown in the figure). For comparison, we also apply a gradient based update based on~\cite{cappe-hmm}. The gradient based updates are extremely unstable and best final result obtained is $83.84$ (also not shown in the figure).

\begin{figure*}[t!]
\vspace{-0.5cm}
\begin{center}
    \subfigure[]{\includegraphics[width=0.32\textwidth]{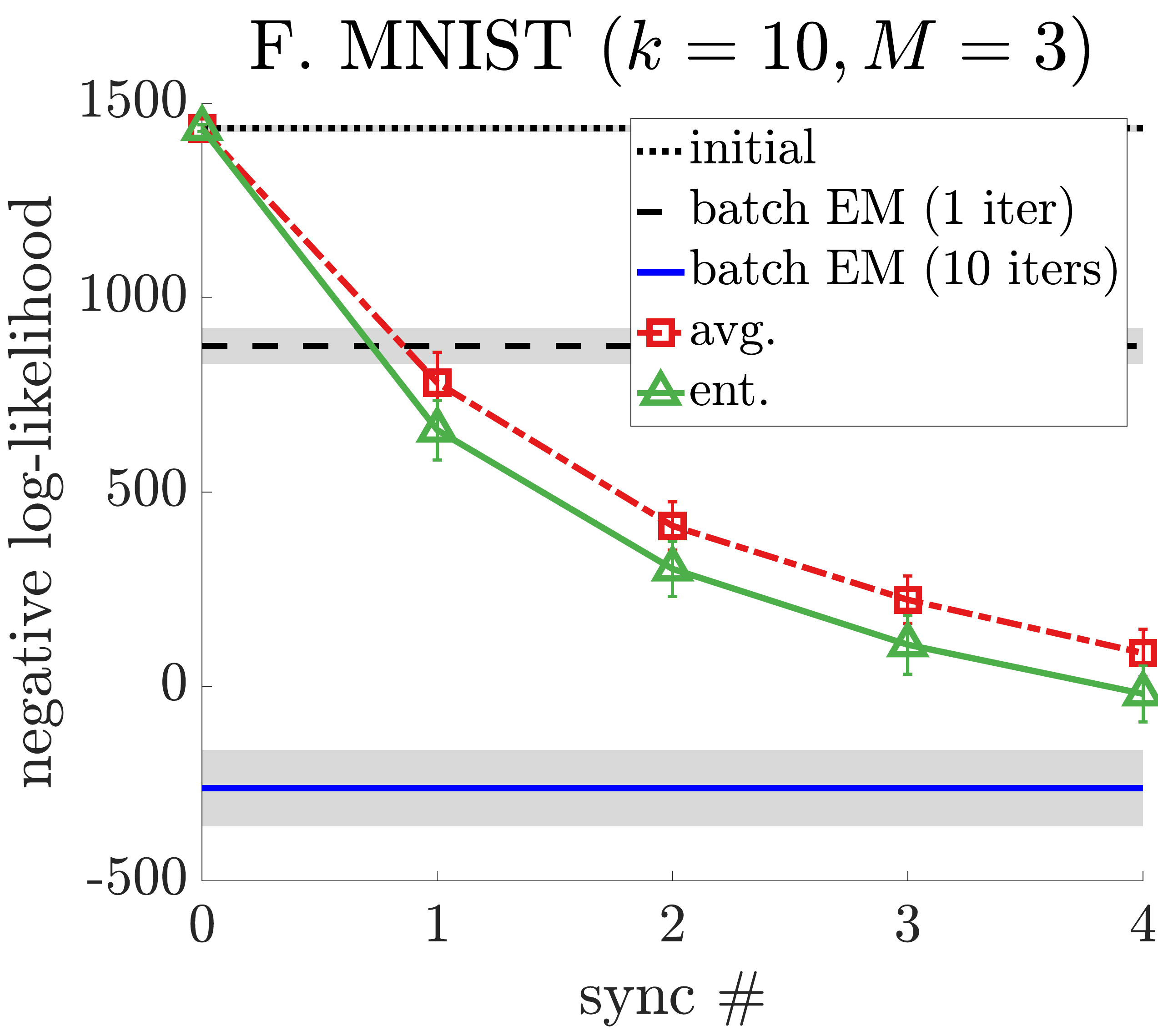}}\,
    \subfigure[]{\includegraphics[width=0.32\textwidth]{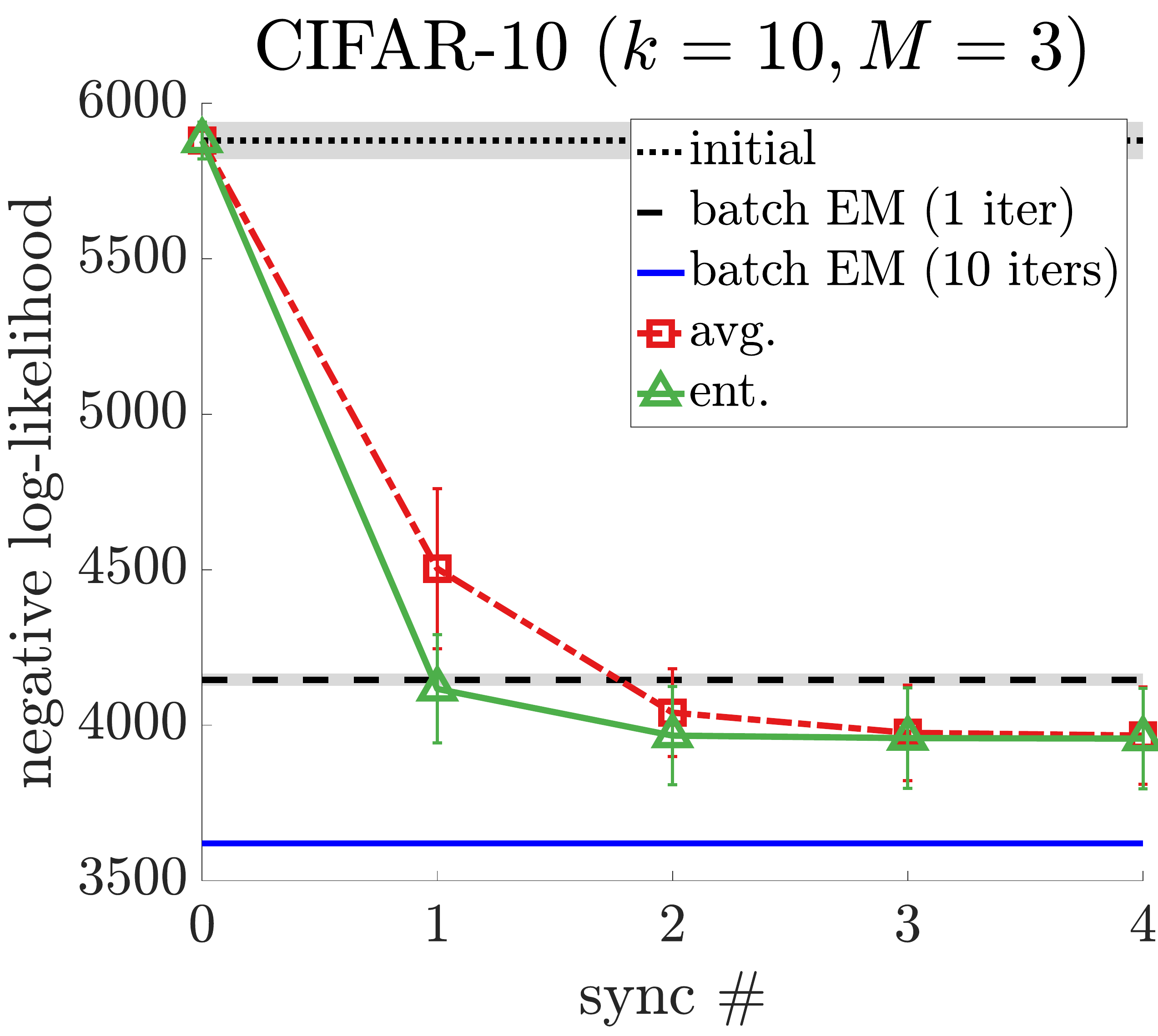}}\,
    \subfigure[]{\includegraphics[width=0.32\textwidth]{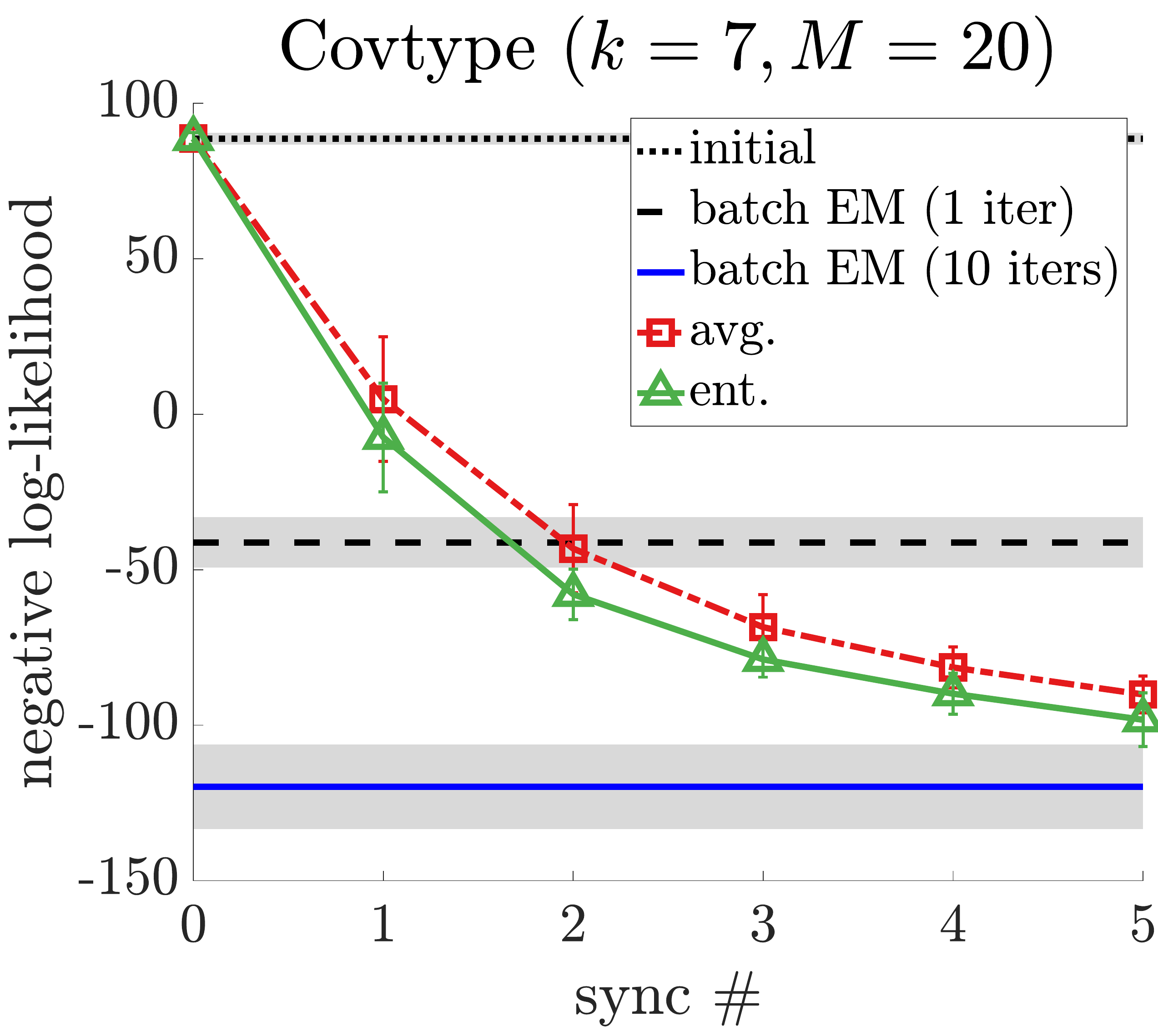}}\,
    \vspace{-0.45cm}
     \caption{Combining models: results using simple averaging of the means and covariance matrices (avg.) compared to our proposed model averaging (ent.) on (a) Fashion MNIST, (b) CIFAR-10, and (c) Covtype datasets. The number of mixtures $k$ as well as the number of machines $M$ is shown on top of each plot. The sync step happens after every \num{5000} iterations of online EM on each machine.}\label{fig:dist}
\end{center}
    \vspace{-0.75cm}
\end{figure*}

\subsection{KALMAN FILTER}
To validate the correctness of the updates for Kalman filters, we consider online estimation of the parameters of a Kalman filter with hidden state vector of dimension $5$ and observation vector of dimension $10$. We assume that the noise covariances $Q$ and $R$ are known and consider estimating the remaining parameters, i.e. $\{\pi_1, V, A, C\}$. We apply the batch EM updates as well as the online updates with parameters $(\eta_0, \beta) = (1.0,0.9)$. The results are shown in Figure~\ref{fig:results}-b. Again, the online updates outperform the solution of one batch EM after around $40$ iterations and converge to a solution with a loss close to $10$ batch EM updates. Moreover, the updates are stable wrt the initial learning rate $\eta_0$. The final value of the negative log-likelihood of the model obtained using $\eta_0 = 0.1$ and $\eta_0 = 10.0$ are $89.36$ and $75.90$, respectively.

\subsection{COMPOUND DIRICHLET DISTRIBUTION}
We consider online estimation of a compound Dirichlet distribution~\cite{gupta}. In this case, the EM updates for the model do not have a closed-form solution and therefore, numerical techniques such as Newton's~\cite{nocedal} method are used for performing the updates. The details are given in Appendix~E. As a result, the relative entropy inertia term also does not admit a closed-form and thus, we use the sampling approximation of the inertia term.
(See end of Section \ref{s:combo}.) We consider \num{2000} samples from a $10$ dimensional model and perform batch EM updates as well as online updates with mini-batch size equal to $100$. In order to form the inertia term, we use \num{2000} samples from the model and use parameters $(\eta_0, \beta) = (1.0, 0.9)$ for the learning rate. We use Newton's method for optimization. The result is shown in Figure~\ref{fig:results}-c. As can be seen, the online EM algorithm effectively learns the model parameters. The updates are stable for a lower or higher initial learning rate. The final negative log-likelihood values for $\eta_0 = 0.5$ and $\eta_0 =2.0$ are $10.09$ and $10.08$, respectively (results not shown in the figure).

\vspace{-0.1cm}
\subsection{Distributed Training of Gaussian Mixtures}

We conduct experiments on combining the parameters of Gaussian mixture models in a distributed setting. For this set of experiments, we consider the Fashion MNIST\footnote{\url{https://github.com/zalandoresearch/fashion-mnist}} (dim=784), CIFAR-10\footnote{\url{https://www.cs.toronto.edu/~kriz/cifar.html}} (dim = \num{3072}), and Covtype\footnote{\url{https://archive.ics.uci.edu/ml/datasets/covertype}} (dim = 54) datasets. The number of machines for each dataset is set to $M=3$, $3$, and $20$, respectively. To achieve an equal number of splits across machines, we consider a subset of 60K, 60K, and 500K points from each dataset, respectively. We set the number of mixtures $k$ equal to the number of classes, which amounts $k=10$ for Fashion MNIST, $k = 10$ for CIFAR-10, and $k=7$ for Covtype dataset. We use $(\eta_0, \beta) = (0.05, 0.5)$ for all datasets.

At each trial, all the machines are initialized with the same set of initial parameters. We consider synchronous updates where the parameters of all machines are combined into a single set of parameters at the end of each step and propagated back to each individual machine for the next step. Each machine receives a different set of \num{5000} observations at each step and the process is repeated until one pass over the whole dataset is achieved.

We compare two parameter combining strategies: 1) simple
averaging where mixture weights as well as the expected
values of the conditional sufficient statistics (i.e. means
and covariances of each mixture component) are averaged
over all machines~\citep{open}, and 2) our entropic
combining of parameters as in~\eqref{eq:comb} where we
average the complete data sufficient statistics. The
results are shown in Figure~\ref{fig:dist}. As can be seen,
the divergence based combining of the model provides a
consistently better performance. Specifically, it shows faster convergence to a better solution. Additionally, on the Fashion MNIST and Covtype datasets, the final combined model has a lower negative log-likelihood using our \mbox{divergence} based combining.

\section{CONCLUSION AND FUTURE WORK}
We provided an alternative view of the online EM algorithm~\citep{cappe}
based on divergences between the models. Our new formulation casts new insight on the
algorithm and facilitates finding the updates for more
complex models without the need for identifying the
sufficient statistics. The divergences between models
that we use as inertia terms are interesting in their own
right and are the most important outcome of this research. These divergences can be
approximated in cases where the EM updates do not have a closed-form.
Also, the divergences between the models lead to a new technique for combining models
which is useful in distributed settings.

There are a number of intriguing open problems coming out of the current work.
All our divergences are based on joint relative entropies where the new
model is always in the second argument. In online learning, the new
model is typically in the first argument (see e.g. \citep{eg}).
Also in the context of reinforcement learning \citep{gergeley}, the alternate
joint entropies for HMMs (with the new parameters as the first argument)
have been used effectively.
The alternate relative entropies appear to be more stable. Therefore, the
question is whether there is a use of the alternate for
producing useful updates for minimizing the negative log-likelihood of
hidden variable problems.

\bibliography{refs}

\begin{thebibliography}{35}
\providecommand{\natexlab}[1]{#1}
\providecommand{\url}[1]{\texttt{#1}}
\expandafter\ifx\csname urlstyle\endcsname\relax
  \providecommand{\doi}[1]{doi: #1}\else
  \providecommand{\doi}{doi: \begingroup \urlstyle{rm}\Url}\fi

\bibitem[Amari \& Nagaoka(2007)Amari and Nagaoka]{amari}
Amari, S.-i. and Nagaoka, H.
\newblock \emph{Methods of information geometry}, volume 191.
\newblock American Mathematical Soc., 2007.

\bibitem[Baldi \& Chauvin(1994)Baldi and Chauvin]{baldi}
Baldi, P. and Chauvin, Y.
\newblock Smooth on-line learning algorithms for hidden {M}arkov models.
\newblock \emph{Neural Computation}, 6\penalty0 (2):\penalty0 307--318, 1994.

\bibitem[Blei et~al.(2017)Blei, Kucukelbir, and McAuliffe]{blei}
Blei, D.~M., Kucukelbir, A., and McAuliffe, J.~D.
\newblock Variational inference: A review for statisticians.
\newblock \emph{Journal of the American statistical Association}, 112\penalty0
  (518):\penalty0 859--877, 2017.

\bibitem[Bregman(1967)]{bregman}
Bregman, L.~M.
\newblock The relaxation method of finding the common point of convex sets and
  its application to the solution of problems in convex programming.
\newblock \emph{USSR computational mathematics and mathematical physics},
  7\penalty0 (3):\penalty0 200--217, 1967.

\bibitem[Capp\'{e}(2011)]{cappe-hmm}
Capp\'{e}, O.
\newblock Online {EM} algorithm for hidden {M}arkov models.
\newblock \emph{Journal of Computational and Graphical Statistics}, 20\penalty0
  (3):\penalty0 728--749, 2011.

\bibitem[Capp{\'e} \& Moulines(2009)Capp{\'e} and Moulines]{cappe}
Capp{\'e}, O. and Moulines, E.
\newblock On-line expectation--maximization algorithm for latent data models.
\newblock \emph{Journal of the Royal Statistical Society: Series B (Statistical
  Methodology)}, 71\penalty0 (3):\penalty0 593--613, 2009.

\bibitem[Capp{\'e} et~al.(1998)Capp{\'e}, Buchoux, and Moulines]{cappe-quasi}
Capp{\'e}, O., Buchoux, V., and Moulines, E.
\newblock Quasi-{N}ewton method for maximum likelihood estimation of hidden
  {M}arkov models.
\newblock In \emph{IEEE International Conference on Acoustics, Speech, and
  Signal Processing}, volume~4, pp.\  IV--2265, 1998.

\bibitem[Cichocki et~al.(2009)Cichocki, Zdunek, Phan, and Amari]{nmf}
Cichocki, A., Zdunek, R., Phan, A.-H., and Amari, S.-i.
\newblock \emph{Nonnegative Matrix and Tensor Factorizations: Applications to
  Exploratory Multi-Way Data Analysis and Blind Source Separation}.
\newblock Wiley, first edition, 2009.

\bibitem[Collings et~al.(1994)Collings, Krishnamurthy, and Moore]{collings}
Collings, I., Krishnamurthy, V., and Moore, J.
\newblock On-line identification of hidden {M}arkov models via recursive
  prediction error techniques.
\newblock \emph{IEEE Transactions on Signal Processing}, 42:\penalty0
  3535--3539, 1994.

\bibitem[Do \& Batzoglou(2008)Do and Batzoglou]{do}
Do, C.~B. and Batzoglou, S.
\newblock What is the expectation maximization algorithm?
\newblock \emph{Nature biotechnology}, 26\penalty0 (8):\penalty0 897, 2008.

\bibitem[Florez-Larrahondo et~al.(2005)Florez-Larrahondo, Bridges, and
  Hansen]{florez}
Florez-Larrahondo, G., Bridges, S., and Hansen, E.~A.
\newblock Incremental estimation of discrete hidden {M}arkov models based on a
  new backward procedure.
\newblock In \emph{Proceedings of the 20th national conference on Artificial
  intelligence-Volume 2}, pp.\  758--763. AAAI Press, 2005.

\bibitem[Garg \& Warmuth(2003)Garg and Warmuth]{garg}
Garg, A. and Warmuth, M.~K.
\newblock Inline updates for {HMM}s.
\newblock In \emph{INTERSPEECH}, 2003.

\bibitem[Ghahramani \& Hinton(1996)Ghahramani and Hinton]{ghahramani}
Ghahramani, Z. and Hinton, G.~E.
\newblock Parameter estimation for linear dynamical systems.
\newblock Technical report, CRG-TR-96-2, 1996.

\bibitem[Gupta et~al.(2011)Gupta, Chen, et~al.]{gupta}
Gupta, M.~R., Chen, Y., et~al.
\newblock Theory and use of the {EM} algorithm.
\newblock \emph{Foundations and Trends{\textregistered} in Signal Processing},
  4\penalty0 (3):\penalty0 223--296, 2011.

\bibitem[Hiriart-Urruty \& Lemaréchal(2001)Hiriart-Urruty and
  Lemaréchal]{urruty}
Hiriart-Urruty, J.-B. and Lemaréchal, C.
\newblock \emph{Fundamentals of Convex Analysis}.
\newblock Springer-Verlag Berlin Heidelberg, first edition, 2001.

\bibitem[Kivinen \& Warmuth(1997)Kivinen and Warmuth]{eg}
Kivinen, J. and Warmuth, M.~K.
\newblock Exponentiated gradient versus gradient descent for linear predictors.
\newblock \emph{Inf. Comput.}, 132\penalty0 (1):\penalty0 1--63, 1997.

\bibitem[Kontorovich et~al.(2013)Kontorovich, Nadler, and Weiss]{kontorovich}
Kontorovich, A., Nadler, B., and Weiss, R.
\newblock On learning parametric-output {HMM}s.
\newblock In \emph{International Conference on Machine Learning}, pp.\
  702--710, 2013.

\bibitem[Krishnamurthy \& Moore(1993)Krishnamurthy and Moore]{krishnamurthy}
Krishnamurthy, V. and Moore, J.~B.
\newblock On-line estimation of hidden {M}arkov model parameters based on the
  {K}ullback-{L}eibler information measure.
\newblock \emph{IEEE Transactions on signal processing}, 41\penalty0
  (8):\penalty0 2557--2573, 1993.

\bibitem[LeGland \& Mevel(1997)LeGland and Mevel]{legland}
LeGland, F. and Mevel, L.
\newblock Recursive estimation in hidden {M}arkov models.
\newblock In \emph{Decision and Control, 1997., Proceedings of the 36th IEEE
  Conference on}, volume~4, pp.\  3468--3473. IEEE, 1997.

\bibitem[McLachlan \& Krishnan(2008)McLachlan and Krishnan]{embook}
McLachlan, G.~J. and Krishnan, T.
\newblock \emph{{The {EM} Algorithm and Extensions (Wiley Series in Probability
  and Statistics)}}.
\newblock Wiley-Interscience, 2 edition, 2008.

\bibitem[Mizuno et~al.(2000)Mizuno, Watanabe, Ueki, Amano, Takimoto, and
  Maruoka]{mizuno}
Mizuno, J., Watanabe, T., Ueki, K., Amano, K., Takimoto, E., and Maruoka, A.
\newblock On-line estimation of hidden {M}arkov model parameters.
\newblock In \emph{International Conference on Discovery Science}, pp.\
  155--169. Springer, 2000.

\bibitem[Mongillo \& Deneve(2008)Mongillo and Deneve]{mongillo}
Mongillo, G. and Deneve, S.
\newblock Online learning with hidden {M}arkov models.
\newblock \emph{Neural computation}, 20\penalty0 (7):\penalty0 1706--1716,
  2008.

\bibitem[Neal \& Hinton(1998)Neal and Hinton]{neal}
Neal, R.~M. and Hinton, G.~E.
\newblock A view of the {EM} algorithm that justifies incremental, sparse, and
  other variants.
\newblock In \emph{Learning in graphical models}, pp.\  355--368. Springer,
  1998.

\bibitem[Neu et~al.(2017)Neu, Jonsson, and G{\'{o}}mez]{gergeley}
Neu, G., Jonsson, A., and G{\'{o}}mez, V.
\newblock A unified view of entropy-regularized markov decision processes.
\newblock \emph{arXiv preprint {\tt http://arxiv.org/abs/1705.07798}}, 2017.

\bibitem[Nocedal \& Wright(2006)Nocedal and Wright]{nocedal}
Nocedal, J. and Wright, S.~J.
\newblock \emph{Nonlinear Equations}.
\newblock Springer, 2006.

\bibitem[Rabiner(1989)]{rabiner}
Rabiner, L.~R.
\newblock A tutorial on hidden {M}arkov models and selected applications in
  speech recognition.
\newblock \emph{Proceedings of the IEEE}, 77\penalty0 (2):\penalty0 257--286,
  1989.

\bibitem[Sanderson \& Curtin(2017)Sanderson and Curtin]{open}
Sanderson, C. and Curtin, R.
\newblock An open source {C}++ implementation of multi-threaded {G}aussian
  mixture models, k-means and expectation maximisation.
\newblock In \emph{IEEE Int. Conf. on Signal Processing and Communication
  Systems (ICSPCS)}, pp.\  1--8, 2017.

\bibitem[Sato(2000)]{sato}
Sato, M.
\newblock Convergence of on-line em algorithm.
\newblock \emph{7th International Conference on Neural Information Processing},
  1:\penalty0 476--481, 01 2000.

\bibitem[Singer \& Warmuth(1997)Singer and Warmuth]{singer-hmm}
Singer, Y. and Warmuth, M.~K.
\newblock Training algorithms for hidden {M}arkov models using entropy based
  distance functions.
\newblock In \emph{Advances in Neural Information Processing Systems}, pp.\
  641--647, 1997.

\bibitem[Singer \& Warmuth(1999)Singer and Warmuth]{singer}
Singer, Y. and Warmuth, M.~K.
\newblock Batch and on-line parameter estimation of {G}aussian mixtures based
  on the joint entropy.
\newblock In \emph{Proceedings of Advances in Neural Information Processing
  Systems}, pp.\  578--584, 1999.

\bibitem[Titterington(1984)]{titterington}
Titterington, D.~M.
\newblock {Recursive Parameter Estimation Using Incomplete Data}.
\newblock \emph{Journal of the Royal Statistical Society, Series B},
  46\penalty0 (2):\penalty0 257--267, 1984.

\bibitem[Wainwright et~al.(2008)Wainwright, Jordan, et~al.]{wainwright}
Wainwright, M.~J., Jordan, M.~I., et~al.
\newblock Graphical models, exponential families, and variational inference.
\newblock \emph{Foundations and Trends{\textregistered} in Machine Learning},
  1\penalty0 (1--2):\penalty0 1--305, 2008.

\bibitem[Wang \& Zhao(2006)Wang and Zhao]{wang}
Wang, S. and Zhao, Y.
\newblock Almost sure convergence of {T}itterington's recursive estimator for
  mixture models.
\newblock \emph{Statistics \& probability letters}, 76\penalty0 (18):\penalty0
  2001--2006, 2006.

\bibitem[Welch \& Bishop(1995)Welch and Bishop]{welch}
Welch, G. and Bishop, G.
\newblock An introduction to the {K}alman filter.
\newblock Technical report, University of North Carolina at Chapel Hill, 1995.

\bibitem[Wu(1983)]{wu}
Wu, C.~J.
\newblock On the convergence properties of the {EM} algorithm.
\newblock \emph{The Annals of statistics}, pp.\  95--103, 1983.

\end{thebibliography}
\bibliographystyle{icml2020}

\clearpage

\newpage
\appendix
\section{PROOF OF THEOREM~\ref{obs:cappe}}

\begin{proof} The online E and M steps in~\citep{cappe} are defined respectively as
 \begin{align}
     &\hat{Q}_{t+1}(\Thett)  = \hat{Q}_{t}(\Thett)\nonumber\\
     &\,\,+ \eta_{t} \,\big(\EE_{P(h|\,v_{t+1},\Theta^t)}[\log P(h, v_{t+1}|\,\Thett)] - \hat{Q}_{t}(\Thett)\big)\label{eq:cappe-e}\\ 
     & \Thett^{t+1} = \arg\max_{\Thett}\, \hat{Q}_{t+1}(\Thett) = \arg\min_{\Thett}\, -\hat{Q}_{t+1}(\Thett) \label{eq:cappe-m}
 \end{align}
 where
\[
     \hat{Q}_t(\Thett) = \EE_{P(h, v|\,\Theta^t)} \big[\log P(h, v|\,\Thett)\big]\, .
\]
The relative entropy divergence between the hidden variable models can be written as
\begin{align}
    \Delta(\Theta^t, \Thett) = &  \int_{h,\,v} P(h,v|\,\Theta^t) \log\frac{P(h,v|\,\Theta^t)}{P(h,v|\,\Thett)} \nonumber\\
    = & \int_{h,\,v} P(h,v|\,\Theta^t) \log P(h,v|\,\Theta^t) - \hat{Q}_t(\Thett) \label{eq:cappe-inertia}
    \end{align}
    Note that the first term in~\eqref{eq:cappe-inertia} does not depend on $\Thett$. Additionally, for the EM upper-bound, we have
    \begin{align}
        \UP_{\Theta}&(\Thett|\,v_{t+1})  = \EE_{P(h|\,v_{t+1}, \Theta^t)} \big[
        \log P(h|v_{t+1},\Theta)\big]\nonumber\\
    &-
    \EE_{P(h|\,v_{t+1}, \Theta^t)} \big[\log P(h, v_{t+1}|\,\Thett)\big] \label{eq:cappe-up}
\end{align}
where again the first term does not depend on $\Thett$ and the second term corresponds to the negative of the expectation in~\eqref{eq:cappe-e}. Comparing~\eqref{eq:online-em} with $\eta^{(t)} = \sfrac{\eta_{t}}{(1 - \eta_{t})}$ and ignoring the constant yields the same updates given in~\eqref{eq:cappe-e} and~\eqref{eq:cappe-m}.
\end{proof}

\section{BREGMAN DIVERGENCE AND EXPONENTIAL FAMILY}
In this section, we review Bregman divergence and exponential family as well as the required lemmas for deriving the updates.

For a real-valued continuously-differentiable and strictly convex function $F:\, \RR^d \rightarrow \RR$, the Bregman divergence~\cite{bregman,nmf} $\Delta_F(\zt, \z)$ between $\zt$ and $\z$ is defined as
$$
\Delta_F(\zt, \z) =  F(\zt) - F(\z) -  f(\z)\cdot (\zt - \z)\, ,
$$
where $f(\z) \defeq \nabla_{\z}F(\z)$. The gradient wrt the first arguments take the form
\begin{align*}
     \nabla_{\zt}\, \Delta_F(\zt,\z) & = f(\zt) - f(\z)\, ,\\
      \intertext{while the gradient wrtthe second argument becomes}
        \nabla_{\z}\, \Delta_F(\zt,\z) & = -\nabla^2F(\z) (\zt - \z)\, .
    \end{align*}
    The Fenchel dual~\cite{urruty} of the function $F$ is defined as
    $$
    F^*(\zd) = \sup_{\zp}\,\, \zp \cdot \z^* - F(\zp)\, .
    $$
    Assuming that the supremum is achieved at $\z$, we have the following relation between variables $\z$ and $\zd$
     $$\z^* = f(\z) \text{\, , \,\,\,} \z = f^*(\z^*) \text{\, , \, and \,\,} f^* = f^{-1}\, ,$$
    where $f^*(\zd) \defeq \nabla_{\zd} F^*(\zd)$. Note that as a result of convexity of $F^*$, we can form the dual Bregman divergence using $F^*$ as the generating convex function. The following equality holds for pairs of dual variables $(\z, \zd)$ and $(\zt,\zt^*)$
    \begingroup
    \allowdisplaybreaks
    \begin{align*}
        \Delta_F(\zt, \z)
        & = F(\zt) - F(\z) - f(\z)\cdot (\zt - \z)\\
        &\qquad \, + \, f(\zt)\cdot \zt - f(\zt)\cdot \zt\\
        & = \underbrace{-F(\z) + f(\z)\cdot \z}_{F^*(\z^*)}  + \underbrace{F(\zt) - f(\zt)\cdot \zt}_{-F^*(\zt^*)}\\
        &\qquad - \underbrace{\zt\cdot(f(\z) - f(\zt))}_{f^*(\zt^*)\cdot(\z^* - \zt^*)}\\
        & = \Delta_{F^*}(\z^*, \zt^*)\, .
    \end{align*}
    \endgroup
       Note that the order of variables is reversed when switching to the dual divergence. Additionally, using the definition of the dual function, we have
        $$
    \Delta_F(\zt,\z) = \Delta_{F^*}(\zd,\zt^*) = F(\zt) + F^*(\z^*) - \zt \cdot \z^*\, .
    $$
    The following lemmas for combining Bregman divergences are useful for our discussion of our EM updates.
    \begin{lemma}{\textbf{Forward Combination}}
    \label{lem:forward-comb}
    Let $\{\alpha_i\}_{i=1}^N$ where $\alpha_i \in \RR_+$ and $\sum_i \alpha_i > 0$. We have
        $$
    \z_{opt} = \arg\min_{\zt}
    \sum_i \alpha_i\, \Delta_F(\zt, \z_i)
    = f^{*}\Big(\frac{\sum_i \alpha_i\, f(\z_i)}{\sum_i \alpha_i}\Big)\, .
    $$
    \end{lemma}
\begin{proof}
Taking the derivative of the objective function wrt $\zt$ and using the gradient property of the Bregman divergence with respect to the first argument, we have
$$
    \sum_i \alpha_i\, \Big(f(\zt) - f(\z_i)\Big) = 0\, ,
$$
which yields
    \begin{align*}
        &\big(\sum_i \alpha_i\big)\, f(\z_{opt})  = \sum_i \alpha_i\, f(\z_i)\, ,\\
        & \text{\,\,\, or\,\,\,\,\,} \z_{opt} = f^{-1}\Big(\frac{\sum_i \alpha_i\, f(\z_i)}{\sum_i \alpha_i}\Big)\, .
    \end{align*}
    Using the fact that $f^{-1} = f^*$ completes the proof.
\end{proof}
    \begin{corollary}{\textbf{Forward Triangular Equality}}
        \begin{multline*}
        \sum_i \alpha_i\, \Delta(\zt, \z_i)
    \,-\, \sum_i \alpha_i\, \Delta(\z_{opt}, \z_i)
   \\
            = \big(\sum_i \alpha_i\big) \,\Delta(\zt, \z_{opt})\, .
        \end{multline*}
    \end{corollary}
     \begin{lemma}{\textbf{Backward Combination}}
    \label{lem:backward-comb}
    Let $\{\alpha_i\}_{i=1}^N$ where $\alpha_i \in \RR_+$ and $\sum_i \alpha_i > 0$. We have
         \begin{align*}
             & \z^*_{opt}  = \arg\min_{\zt^*} \bigg[
                 \sum_i \alpha_i\, \Delta_{F^*}(\z^*_i, \zt^*)\\
                 &\,\,= \sum_i \alpha_i \big(F^*(\z^*_i) + F(\zt) - \zt\, \cdot\, \z^*_i \big)\bigg] = \frac{\sum_i \alpha_i\, \z^*_i}{\sum_i \alpha_i}\, .
         \end{align*}
     \end{lemma}
    \begin{proof}
        Taking the derivative of the objective function wrt $\zt^*$ and using the gradient property of the Bregman divergence with respect to the first argument, we have
    $$
        -\sum_i \alpha_i\, \Big(\nabla^2 F^*(\zt^*)\, \big(\z_i^* - \zt^*\big)\Big) = 0\, .
    $$
Using the fact that $\nabla^2 F^*(\zt^*) \succeq 0$ and rearranging the terms concludes the proof.
    \end{proof}
    \begin{corollary}{\textbf{Backward Triangular Equality}}
        \begin{multline*}
        \sum_i \alpha_i\, \Delta_{F^*}(\z_i^*, \zt^*)
            \,-\, \sum_i \alpha_i\, \Delta(\z^*_i, \z^*_{opt})\\
            = \big(\sum_i \alpha_i\big) \,\Delta(\z^*_{opt},\zt^*)\, .
        \end{multline*}
    \end{corollary}
    In some cases, the value of $F^*(\z^*_i)$ becomes unbounded (see Appendix B). However, we can still apply Lemma~\ref{lem:forward-comb} and~\ref{lem:backward-comb} by dropping the $F^*(\z^*_i)$ terms from the objective.
    \begin{lemma}{\textbf{Partial Combination}}
    Let $\{\alpha_i\}_{i=1}^N$ where $\alpha_i \in \RR_+$ and $\sum_i \alpha_i > 0$. We have
        \begin{align*}
            \z_{opt} &= \arg\min_{\zt} \sum_i \alpha_i \big(F(\zt)-
    \zt\, \cdot\, \z^*_i \big) \\
            &= f^{*} \Big(\frac{\sum_i \alpha_i\, f(\z_i)}{\sum_i \alpha_i}\Big)\, ,
        \end{align*}
    i.e.
        $$
            \z^*_{opt} = \frac{\sum_i \alpha_i\, \z^*_i}{\sum_i \alpha_i}\, .
            $$
    \end{lemma}
    \begin{corollary}
    \begin{multline*}
	\sum_i \alpha_i  \big(F(\zt)- \zt \cdot \z_i^* \big)  +
       \big(\sum_i \alpha_i\big) \,F^*(\z^*_{opt})\\
	= \big(\sum_i \alpha_i\big) \;\Delta_F(\zt, \z_{opt})\, .
    \end{multline*}
    \end{corollary}

\section{COMPOUND DIRICHLET DISTRIBUTION}

Compound Dirichlet distribution (also referred to as P\'{o}lya distribution)~\cite{gupta} is commonly used to model distribution over topics. A topic entails a distribution over words. More specifically, the compound Dirichlet distribution includes a non-negative parameter vector $\alpha > 0$ corresponding to a Dirichlet distribution over topics. The sampling process consists of sampling a topic $h_n$ for the $n$-th document from the Dirichlet distribution. The component $h_{n,i}$ corresponds to the probability of sampling the $i$-th word. Next, a set of iid samples $v_n$ are drawn from the topic. That is, $v_{n,i}$ denotes the frequency of the $i$-th word and $\sum_i v_{n,i}$ is the total number of words in the $n$-th document. Note that the sampled topics are hidden and only the set of documents are given. The set of model parameters equals to $\Theta = \{\alpha\}$.

The join distribution over the hidden topics and visible documents can be written as
$$
P(h,v|\, \Theta) = \prod_n  \frac{\Gamma(\alpha_0)}{\prod_j \Gamma(\alpha_j)}\, \frac{(\sum_j v_{n,j})!}{\prod_j (v_{n,j}!)}\, \prod_i\, h_{n,i}^{\alpha_i + v_{n,i} -1}\, .
$$
where $\alpha_0 = \sum_j \alpha_j$ and $\Gamma(\cdot)$ is the gamma function.

The marginal probability of the documents can be calculated by integrating out the hidden topics, that is,
$$
P(v|\, \Theta) = \prod_n \frac{\big(\sum_j v_{n,j}\big)! \, \Gamma(\alpha_0)\,\prod_j \Gamma\big(\alpha_j + v_{n,j}\big)}{\big(\prod_j v_{n,j}!\big)\, \big(\prod_j \Gamma(\alpha_j)\big)\, \Gamma\big(\sum_j (\alpha_j + v_{n,j})\big)}\, .
$$
The EM upper-bound can be written as
\begin{align*}
    & \UP_{\Theta} (\Thett) = \sfrac{1}{N}\,\sum_n P(h|\, v_n, \alpha)\, \log P(v_n, h|\, \widetilde{\alpha})\\
        &\,\,=\sfrac{1}{N}\,\sum_n P(h|\, v_n, \alpha)\, \log P(h|\, \widetilde{\alpha})\\
        &\,\,= \log\left(\frac{\Gamma(\widetilde{\alpha}_0)}{\prod_j \Gamma(\widetilde{\alpha}_j)}\right)\\
        &\,\,\,\,\,\, +  \sfrac{1}{N}\,\sum_n \sum_j P(h|\, v_n, \alpha)\, (\widetilde{\alpha}_j - 1)\log h_j\\
        &\,\,= N\, \log\left(\frac{\Gamma(\widetilde{\alpha}_0)}{\prod_j \Gamma(\widetilde{\alpha}_j)}\right)\\
        & + \sum_n \sum_j (\widetilde{\alpha}_j - 1)\, \bigg(\psi(v_{nj} + \alpha_j) - \psi(\sum_i v_{ni} + \alpha_0) \bigg)\, ,
    \end{align*}
    where\, $\psi(\alpha) \defeq \frac{\partial}{\partial\alpha}\log\Gamma(\alpha)$\, is called the digamma function. The inertia term on the hand
    \begin{align*}
        & \Delta(\Theta, \Thett) = \frac{\Gamma(\alpha_0)}{\Gamma(\widetilde{\alpha}_0)} - \sum_j \frac{\Gamma(\alpha_j)}{\Gamma(\widetilde{\alpha}_j)}\\
        &+ \Bigg[\sum_j (\alpha_j - \widetilde{\alpha}_j) \sum_v  \frac{\Gamma(\alpha_0)}{\prod_i \Gamma(\alpha_i)} \frac{(\sum_i v_i)!}{\prod_i (v_i!)} \frac{\Gamma(\alpha_0 + \sum_i v_i)}{\prod_i \Gamma(\alpha_i + v_i)}\,\\
        & \qquad \times \bigg(\psi(v_j + \alpha_j) - \psi(\sum_i v_i + \alpha_0) \bigg)\Bigg]\, ,
    \end{align*}
involves summing over all possible combinations of $v$ and therefore, is intractable. Alternatively, we can use the approximate form of the upper-bound to perform the updates.

A standard approach to minimize the upper-bound is the Newton's method~\cite{nocedal}, which requires calculating the gradient and the Hessian matrix. The gradient of the upper-bound can be written as
\begin{multline*}
     \nabla_{\alpha_i}\UP_{\Theta}(\Thett|\, \Vis)  = \psi(\widetilde{\alpha}_0) - \psi(\widetilde{\alpha}_i)\\
     + \sfrac{1}{N}\, \sum_n \big(\psi(\sum_j v_{nj} + \alpha_0) - \psi(v_{ni} + \alpha_i)\big)\, .
    \end{multline*}
    The Hessian is
        $$
        H =  \bigg(\psi_1(\widetilde{\alpha}_0)\bm{1}\bm{1}^\top - \text{diag}\big[\psi_1(\widetilde{\alpha}_1),\ldots,\psi_1(\widetilde{\alpha}_d)\big]\bigg)\, ,
        $$
        where\, $\psi_1(\alpha) \defeq \frac{\partial}{\partial\alpha}\psi(\alpha)$\, is called the trigamma function.


\end{document}